\documentclass{article}

\usepackage{arxiv}

\usepackage[utf8]{inputenc} 
\usepackage[T1]{fontenc}    
\usepackage{hyperref}       
\usepackage{url}            
\usepackage{booktabs}       
\usepackage{natbib}
\usepackage{amsfonts}       
\usepackage{nicefrac}       
\usepackage{microtype}      
\usepackage{lipsum}
\usepackage{graphicx}
\graphicspath{ {./images/} }

\usepackage{amsmath,amssymb}
\usepackage{tikz}
\usepackage{subcaption}

\usepackage[ruled,vlined]{algorithm2e}

\usepackage[utf8]{inputenc} 
\usepackage[T1]{fontenc}    
\usepackage{hyperref}       
\usepackage{url}            
\usepackage{booktabs}       
\usepackage{amsfonts}       
\usepackage{nicefrac}       
\usepackage{microtype}      
\usepackage{xcolor}         
\usepackage{mathrsfs}

\newcommand{\bA}{\bold{A}}

\renewcommand{\P}{\mathbb{P}}
\newcommand{\Z}{\mathrm{Z}}

\newcommand{\R}{\mathbb{R}}

\newcommand{\E}{ \mathbb{E}}

\usepackage{amsthm}

\newtheorem{theorem}{Theorem}[section] 
\newtheorem{lemma}[theorem]{Lemma}     
\newtheorem{proposition}[theorem]{Proposition} 
\newtheorem{remark}[theorem]{Remark} 
\usepackage{tikz}
\usetikzlibrary{positioning,shapes,arrows.meta}

\title{Mallows Model with Learned Distance Metrics: Sampling and Maximum Likelihood Estimation}

\author{%
  Yeganeh Alimohammadi\\
  University of Southern California\\
  \texttt{yalimoha@usc.edu}\\
  \And
  Kiana Asgari\\
  Stanford University\\
  \texttt{asgkiana@stanford.edu}\\
}

\begin{document}

\maketitle

\begin{abstract}

 \textit{Mallows model} is a widely-used probabilistic framework for learning from ranking data, with applications ranging from recommendation systems and voting to aligning language models with human preferences~\cite{chen2024mallows, kleinberg2021algorithmic, rafailov2024direct}. Under this model, observed rankings are noisy perturbations of a central ranking $\sigma$, with likelihood decaying exponentially in distance from $\sigma$, i.e,
$\P (\pi) \propto \exp\big(-\beta \cdot d(\pi, \sigma)\big),$
where $\beta > 0$ controls dispersion and $d$ is a distance function.

Existing methods mainly focus on fixed distances (such as Kendall's $\tau$ distance), with no principled approach to learning the distance metric directly from data. In practice, however, rankings naturally vary by context; for instance, in some sports we regularly see long-range swaps (a low-rank team beating a high-rank one), while in
others such events are rare. 
Motivated by this, we propose a generalization of Mallows model  that learns the distance metric directly from data. Specifically, we focus on $L_\alpha$ distances:
$d_\alpha(\pi,\sigma):=\sum_{i=1} |\pi(i)-\sigma(i)|^\alpha$.

For any $\alpha\geq 1$ and $\beta>0$, we develop a Fully Polynomial-Time Approximation Scheme (FPTAS) to efficiently generate samples that are $\epsilon$- close  (in total variation distance) to the true distribution. Even in the special cases of  $L_1$ and $L_2$, this generalizes prior results that required vanishing dispersion ($\beta\to0$).
Using this sampling algorithm, we propose an efficient Maximum Likelihood Estimation (MLE) algorithm that jointly estimates the central ranking, the dispersion parameter, and the optimal distance metric. We prove strong consistency results for our estimators, and we validate our approach empirically using datasets from sports rankings.

\end{abstract}

\section{Introduction}

Ranking tasks arise in many applications, from online recommendation systems, sports competitions, and voting, to more recently, the alignment of Large Language Models (LLMs) with human preferences \cite{chen2024mallows, kleinberg2021algorithmic, rafailov2024direct}. 
A fundamental challenge is common to all these scenarios: given multiple observed rankings, how do we aggregate them into a single central ranking, \emph{learn} how rankings vary around this center, and \emph{generate} new rankings reflecting observed patterns?

Probabilistic models offer natural tools to address these questions.
Among them, the \textit{Mallows model}\cite{mallows1957non} stands out due to its conceptual simplicity and interpretability. 
Under this model, observed rankings are viewed as noisy perturbations of a central permutation\footnote{Throughout this paper, we use the terms \emph{ranking} and \emph{permutation} interchangeably.
} $\sigma$, with their likelihood decaying exponentially according to a chosen distance function $d$: 
\[
\P(\pi) \;:=\; \frac{\exp\!\bigl(-\beta\, d(\pi,\sigma)\bigr)}{Z},
\]
where \(\beta>0\) controls the dispersion around $\sigma$,  and \(Z\) is the normalizing constant (often referred to as the partition function). 
Due to its interpretability and desirable structural properties, Mallows model has found widespread adoption in recommendation systems, voting predictions, retail assortment optimization, and, more recently, post-training of large language models \cite{chen2024mallows,desir2016assortment,desir2021mallows,lebanon2007non}.

Initially, Mallows~\cite{mallows1957non} proposed this model specifically with Kendall's tau distance, and it was subsequently generalized by Diaconis~\cite{diaconis1988group} to any right-invariant distance metric. 
  In this paper, we focus on the \textit{$L_\alpha$ Mallows model}, defined by the family of $L_\alpha$ distances:
$d_\alpha(\pi, \sigma) := \sum_{i=1}^n |\pi(i) - \sigma(i)|^\alpha.$
Intuitively, the parameter \( \alpha \) controls
the penalty for \emph{long-range} swaps (e.g., changes from rank 1 to rank 100) versus \emph{local} swaps, with larger values of \( \alpha \) penalizing long-range swaps more strongly.

Our contributions focus on two complementary aspects:\\
       \textbf{Parameter Estimation.}  We present an efficient Maximum Likelihood Estimation (MLE) algorithm that jointly estimates the central ranking \(\sigma\), the dispersion parameter \(\beta\), and the distance parameter \(\alpha\).  Unlike previous studies on the Mallows model --which assume a fixed distance metric-- our estimator learns the distance metric jointly with other model parameters, marking the first such result in the literature.  Further, we prove strong consistency of our estimators under general conditions  $\alpha > 0, \beta > 0$  (see Theorem~\ref{thm:main_mle}).        
       \\
       \textbf{Efficient Sampling.} We develop a Fully Polynomial-Time Approximation Scheme (FPTAS) for sampling permutations within \(\epsilon\)-total variation distance of the true \(L_\alpha\)-Mallows distribution, valid for any \(\alpha \geq 1\) and \(\beta > 0\) (Theorem~\ref{thm:main sampling}).
       Our result generalizes previous sampling methods limited to special cases ($L_1$, $L_2$) that required vanishing dispersion (\(\beta \to 0\)) \cite{zhong2021mallows,mukherjee2016estimation}.

We validate our algorithms empirically using sports ranking datasets from American college football and basketball (Section~\ref{sec:empiric}, codebase \cite{asgari2025code}). Our results reveal substantial differences between these two sports: basketball rankings yield a larger $\alpha \approx 1.09$,  reflecting a rapid decay in the probability of large inversions (e.g., low-ranked teams beating top-ranked teams). In contrast, football rankings yield a smaller $\alpha \approx 0.44$, indicating a relatively higher likelihood of large inversions. Such domain-specific variations highlight the importance of learning distance metrics from data rather than fixing them a priori.

\textbf{Related Work}
Prior works on sampling and learning Mallows model assume a fixed distance function, with Kendall's tau distance (counting inversions between permutations) \cite{kendall1938new} being a particularly common choice due to its closed-form expression \cite{doignon2004repeated,fligner1986distance, irurozki2014sampling, lu2014effective,rubinstein2017sorting}. However, it implicitly assumes equal probabilities for all inversions.
For example,  in a web search ranking, swapping the top two results has far more impact on user experience than swapping the 91st and 92nd, yet under Kendall’s tau distance the inversions appear with equal probabilities.
This uniform penalty is misaligned with real-world tasks where the location of ranking errors matters—such as in recommendation systems, sports ranking, and preference learning tasks.

Alternative metrics like Spearman’s footrule ($L_1$) and rank correlation ($L_2$) or more generally the $L_\alpha$ model used in this paper,  aim to mitigate the limitations of Kendall’s tau by assigning smaller penalties to long-range inversions.
Exact sampling under these metrics is known to be NP-hard \cite{bartholdi1989voting}, leading to the development of approximate methods such as Markov chain Monte Carlo (MCMC) algorithms \cite{zhong2021mallows,vitelli2018probabilistic}. While these methods represent significant progress, prior to this work, their computational efficiency was established only under the assumption of vanishing dispersion $(\beta\to0)$.

To address the limitations of Kendall’s tau in another direction, generalizations of the Mallows model have been proposed~\cite{braverman2009sorting,liu2018efficiently,meilua2022recursive}, introducing inversion penalties through graphical models or mixtures of multiple Mallows distributions. While these methods increase expressiveness, they often do so at the cost of substantial model complexity and reduced interpretability. In contrast, the $L_\alpha$-Mallows framework we study offers a principled and tunable way to learn the shape of the distribution through a single metric parameter. Although such extensions are beyond the scope of this work, the $L_\alpha$-Mallows model could similarly be adapted to mixture or hierarchical frameworks to further increase modeling flexibility.

From the perspective of estimation,  extensive research has examined consistency and convergence of Mallows model parameters \cite{vitelli2018probabilistic,awasthi2014learning,feng2022mallows, mao2022learning, busa2019optimal}. In cases where the partition function is intractable (e.g., for $L_1$ and $L_2$), and the maximum likelihood estimators cannot be computed in closed form, estimation methods often rely on Bayesian approaches \cite{vitelli2018probabilistic} or Monte Carlo Maximum Likelihood Estimation (MCMLE) \cite{geyer1992constrained}, which approximates the likelihood using sampling. Among more recent works,  Tang \cite{tang2019mallows} analyzed  MLE under fixed metrics like Kendall's tau and $L_2$, identifying the central ranking and potential upward biased
estimator for the dispersion parameter. Mukherjee and Tagami \cite{mukherjee2016estimation,mukherjee2023inference} provided consistency results for generalized Mallows-type exponential families, though assuming a known central ranking.  Extending this
line of research, our work jointly estimates all model parameters, and establishes strong consistency guarantees for these
estimators.

Alternative probabilistic frameworks include pairwise-comparison models such as Bradley–Terry \cite{bradley1976rank} and its generalization Plackett–Luce \cite{luce1959individual,plackett1975analysis}, which assign latent scores to items and generate rankings sequentially. Their computational simplicity makes them popular in practice \cite{ranking2010label,wainer2023bayesian}, but unlike the Mallows model, the number of parameters grows linearly with the number of items. We include the Plackett–Luce model as a high-parameter benchmark in our empirical evaluation and consistently observe improved performance from our approach (see Section \ref{sec:empiric}).

\section{Parameter Estimation  via Maximum Likelihood}\label{sec:mle}

Our first goal is to estimate the parameters of the $L_\alpha$-Mallow's model, including the central ranking $\sigma_0$, the dispersion parameter $\beta_0$, and the distance function $\alpha_0$.  Suppose we have observed  $m$ rankings   \(\{\pi^{(1)}, \dots, \pi^{(m)}\}\)  i each independently drawn from a Mallows distribution with unknown parameters $(\alpha_0,\beta_0,\sigma_0)$. To estimate these parameters, we use Maximum Likelihood Estimation (MLE),  choosing values that maximize the average log-likelihood of the observed data:
\[
\mathcal{L}_m(\alpha,\beta,\sigma) = \frac{1}{m}\sum_{l=1}^m \log \P_{\alpha,\beta, \sigma}(\pi^{(l)}) = -\frac{\beta}{m} \sum_{l=1}^m d_\alpha(\pi^{(l)}, \sigma) - \log Z_n(\beta,\alpha).
\]
where the probability of observing a ranking $\pi$ is given by
\[
\P_{\alpha,\beta,\sigma}(\pi) = \frac{e^{-\beta d_\alpha(\pi,\sigma)}}{Z_n(\beta,\alpha)},
\quad\text{with}\quad Z_n(\beta,\alpha) = \sum_{\pi \in {S_n}} e^{-\beta d_\alpha(\pi, \sigma)},
\]
and \({S_n}\) is the set of all permutations over $n$ items. \\
With these definitions in place, we now describe a two-step estimation of the MLE (Algorithm~\ref{alg:onepass_mle}).



\subsection{Step 1: Estimating the central ranking}

First, observe that the partition function \(Z_n(\beta, \alpha)\) is invariant to the choice of the central ranking \(\sigma\) (Proposition~\ref{lem:Zn-indep-sigma} in Appendix~\ref{sec: proof-matching}). Thus, minimizing the negative log-likelihood over $\sigma$ effectively reduces to solving a simpler optimization problem:
$\hat{\sigma}_m \in \arg\min_{\sigma \in {S_n}} \sum_{l=1}^m d_\alpha(\pi^{(l)}, \sigma).$

Since the value of $\alpha_0$ is unknown, we instead solve for $L_1$ distance:
\begin{equation}\label{eq:opt-sigma}
    \hat{\sigma}_m \in \arg\min_{\sigma} \sum_{l=1}^m d_1(\pi^{(l)},\sigma).
\end{equation}
This simplification is theoretically justified by the following lemma, which guarantees consistency of the estimated central ranking regardless of the choice of the distance function.

\begin{lemma}\label{lm:unqiue minimizer} Let \(\{\pi^{(1)}, \dots, \pi^{(m)}\}\) be $m$ observed
rankings, which are assumed to be sampled i.i.d from Mallow's model with parameters $(\alpha_0,\beta_0,\sigma_0)$.
Fix any \({\tilde \alpha} > 0\), and let $\hat{\sigma}_m \in \arg\min_{\sigma \in {S_n}} \sum_{l=1}^m d_{\tilde \alpha}(\pi^{(l)}, \sigma).$ Then as $m\to\infty$ the central ranking estimator converges almost surely $\hat{\sigma}_m\overset{\P}{\to}\sigma_0$. Further, \(\sigma_0\) is the unique minimizer of the expected distance, i.e.,
\[
\{\sigma_0\} = \arg\min_{\tilde \sigma \in {S_n}} \Big(\mathbb{E}[d_{\tilde{\alpha}}(\Pi, \tilde\sigma)]\Big),
\]
where $\Pi$ is a random permutation sampled from $\P_{\alpha_0,\beta_0,\sigma_0}$.
\end{lemma}
To see why Lemma~\ref{lm:unqiue minimizer} holds, consider any ranking that differs from  \(\sigma_0\). Such a ranking contains at least one inversion relative to  \(\sigma_0\),  and we show that correcting this inversion strictly reduces the expected distance.
 Repeatedly applying this procedure yields a strictly decreasing sequence of expected distances, ensuring that the true central ranking \(\sigma_0\)  uniquely emerges as the global minimizer. A complete proof appears in Appendix~\ref{sec: proof-mle}.

Crucially, the simplified optimization in \eqref{eq:opt-sigma} is computationally efficient. It corresponds exactly to a \emph{minimum-weight perfect matching} problem on a bipartite graph where the weight of an edge between node \(i\) on the first side and node \(j\) on the second side is given by:
$\sum_{l=1}^m |\pi^{(l)}(i) - j|$. 
This matching problem can be solved efficiently in near-linear time~\cite{duan2010approximating}, giving the central ranking estimator:
\begin{proposition}
\label{prop:update-sigma-matching}
The minimizer \(\hat{\sigma}_m\) can be obtained by solving a minimum-weight perfect matching problem in a bipartite graph, with each side of size \(n\), where the edge weight between node \(i\) on one side and node \(j\) on the other side is:
$\sum_{l=1}^m |\pi^{(l)}(i) - j|.$
\end{proposition}

\subsection{Step 2: Estimating the distance and dispersion parameters}

With the central ranking \(\hat\sigma_m\) fixed, we now estimate the continuous parameters \((\alpha, \beta)\). We formulate this as a two-dimensional root-finding problem, where we seek parameter values that align empirical and model-implied distances.
Specifically,
\[
\Psi_m(\alpha, \beta; \hat\sigma_m) := \left(
-\frac{1}{m} \sum_{i=1}^m d_\alpha(\pi^{(i)}, \hat\sigma_m) + \mathbb{E}[d_\alpha(\Pi, \hat\sigma_m)], \;\;
-\frac{1}{m} \sum_{i=1}^m \dot{d}_\alpha(\pi^{(i)}, \hat\sigma_m) + \mathbb{E}[\dot{d}_\alpha(\Pi, \hat\sigma_m)]
\right),
\]
where $\Pi$ is a random permutation with distribution $\P_{\alpha,\beta,\hat\sigma_m}$, and
$\dot{d}_\alpha(\pi, \sigma) := \sum_{i=1}^n |\pi(i) - \sigma(i)|^\alpha \log |\pi(i) - \sigma(i)|$
denotes the partial derivative of \(d_\alpha\) with respect to \(\alpha\).

The function \(\Psi_m\) corresponds to the gradient of the log-likelihood, with terms rescaled by \(\beta\). We estimate \((\alpha_0, \beta_0)\) by solving:
\[
(\hat\alpha_m, \hat\beta_m) \in \left\{ (\alpha > 0, \beta > 0) :\; \Psi_m(\alpha, \beta; \hat\sigma_m) = (0,\; 0) \right\}.
\]


In practice, a simple zero-order optimization method, such as the differential evolution algorithm~\cite{storn1997differential}, can be used to solve this system efficiently. Note that each iteration of this method requires evaluating the partition function and its derivatives, since we have
\[
\mathbb{E}[d_\alpha(\Pi, \hat\sigma_m)] = -\frac{\partial \log Z_n(\alpha,\beta)}{\partial \beta}, \quad
\beta \cdot \mathbb{E}[\dot{d}_\alpha(\Pi, \hat\sigma_m)] = -\frac{\partial \log Z_n(\alpha,\beta)}{\partial \alpha}.
\]
For large \(n\), we can use a sampling algorithm to approximate these expectations. In Section~\ref{sec:ptas}, we develop a sampling algorithm which yields provable error guarantees for all \(\alpha \geq 1\). 

\begin{algorithm}[ht]
\caption{MLE for \(\{\alpha,\beta,\sigma\}\) in the Mallows Model}
\label{alg:onepass_mle}
\DontPrintSemicolon
\SetKwInOut{Input}{Input}
\Input{%
  Samples \(\{\pi^{(1)}, \dots, \pi^{(m)}\}\);
}

  {\textbf{Step 1:} Solve for \(\hat\sigma_m\) using minimum matching problem.}
  \(\displaystyle 
     \hat\sigma_m 
       := \arg\min_{\sigma}
         \sum_{l=1}^m d_1(\pi^{(l)},\sigma).
   \)
   
\textbf{Step 2:} Estimate parameters \((\hat\alpha_m, \hat\beta_m)\) by solving the system:
\[
(\hat\alpha_m, \hat\beta_m) \leftarrow \text{Solve}\left\{(\alpha,\beta): \Psi_m(\alpha,\beta;\hat\sigma_m)=(0,0)\right\}.
\]
\Return{\(\hat\alpha_m,\hat\beta_m,\hat\sigma_m\).}
\end{algorithm}

\begin{remark} While our estimation procedure works for any $\alpha > 0$, the sampling algorithm developed in Section~\ref{sec:samplingDP} requires $\alpha \geq 1$ to guarantee provable efficiency and approximation accuracy when evaluating the function $\Psi_m$. This constraint is purely computational and does not impact the statistical properties of the estimator or its empirical performance. The sampler simply enables efficient evaluation of  $\Psi_m$, making root-finding for parameter estimation computationally tractable.
\end{remark}

\subsection{Statistical Guarantees}


Under standard assumptions of i.i.d.\ samples from the Mallows distribution, the maximum likelihood estimator proposed in Algorithm~\ref{alg:onepass_mle} jointly estimates the dispersion parameters and central ranking. The result below establishes its statistical guarantees when the score function \(\Psi_m\) is computed exactly: the estimator is consistent, and the continuous parameters converge at the optimal \(\sqrt{m}\) rate with an asymptotically normal distribution.

\begin{theorem}[MLE Consistency and Asymptotic Normality]\label{thm:main_mle}
Suppose the rankings \(\pi^{(1)}, \dots, \pi^{(m)}\) are drawn i.i.d.\ from the Mallows distribution \(\mathbb{P}_{\alpha_0, \beta_0, \sigma_0}\) with true parameters \(\alpha_0, \beta_0, \sigma_0\). Let \((\hat{\alpha}_m, \hat{\beta}_m, \hat{\sigma}_m)\) be the estimators returned by Algorithm~\ref{alg:onepass_mle}. Then, as \(m \to \infty\), the following hold:

\begin{enumerate}
    \item \textit{Consistency of \(\hat{\sigma}_m\):} The estimated central ranking $\hat\sigma_m$ converges almost surely to the true central ranking $\sigma_0$. Moreover, the probability of error admits the finite-sample bound
    \[
    \mathbb{P}\bigl( \hat{\sigma}_m \neq \sigma_0 \bigr)
    \;\le\;
    n! \cdot \exp\left( -\frac{m}{3n^4} \right).
    \]

    \item \textit{Consistency and asymptotic normality of \((\hat{\alpha}_m, \hat{\beta}_m)\):} The estimators \((\hat{\alpha}_m, \hat{\beta}_m)\) converge in probability to \((\alpha_0, \beta_0)\), and further is an efficient estimator; namely,
    \[
    \sqrt{m}
    \begin{pmatrix}
        \hat{\alpha}_m - \alpha_0 \\
        \hat{\beta}_m - \beta_0
    \end{pmatrix}
    \xrightarrow{d}
    \mathcal{N}\left(
        \mathbf{0},\;
        \mathcal I_{\alpha_0, \beta_0}^{-1}
    \right),
    \]
    where $\mathcal I_{\alpha_0, \beta_0}$ is the Fisher information matrix, given by
    \[
    \begin{pmatrix}
        \beta^2\, \mathrm{Var}[\dot{d}_\alpha(\Pi, \text{id})] &
        \beta\, \mathrm{Cov} [\dot{d}_\alpha(\Pi, \text{id}), \; d_\alpha(\Pi, \text{id})] \\[6pt]
        \beta\, \mathrm{Cov}[\dot{d}_\alpha(\Pi, \text{id}), \; d_\alpha(\Pi, \text{id})] &
        \mathrm{Var}[d_\alpha(\Pi, \text{id})]
    \end{pmatrix},
    \]
    where and $\text{id}$ is the identity permutation and $\Pi$ is a random permutation with distribution $\P_{\alpha_0,\beta_0, \text{id}}$.
\end{enumerate}
\end{theorem}

The main challenge for establishing consistency results on the parameters \(\alpha\) and \(\beta\), is that the log-likelihood is not jointly convex in both parameters. However, by carefully analyzing the structure of Mallow's distribution, we show that the log-likelihood gradient admits a \emph{unique}, stable minimizer.
For this purpose, we show that the function $\Psi(\alpha, \beta)$ is identifiable, by showing it is locally invertible (Lemma~\ref{lem:local_invertibility}), and proper (Lemma~\ref{lem:properness}). These conditions, along with uniform convergence of $\Psi_m$ to $\Psi$, allow us to apply Theorem 5.9 of \cite{van2000asymptotic} to conclude consistency and asymptotic normality.
Full technical details appear in Appendix~\ref{sec: proof-mle}.

\begin{remark}
The convergence rate in part \textnormal{(i)} of Theorem~\ref{thm:main_mle} depends on the choice of the distance exponent \(\tilde\alpha\) used in the first step of the algorithm. A direct optimization of the error bound shows that setting \(\tilde\alpha = n/3\) minimizes the upper bound on the probability of misidentifying the true ranking, yielding the improved rate
\[
\mathbb{P}\left( \hat\sigma_m \ne \sigma_0 \right) \le n! \cdot \exp\left( -\frac{m}{54n^2} \right).
\]
However, for ease of implementation and interpretation, we set \(\tilde\alpha = 1\) in all theoretical and empirical results. 
\end{remark}

\paragraph{Approximate computation of \(\Psi_m\).}
When exact evaluation of the gradient  \(\Psi_m\) is not possible --typically due to the intractability of the partition function  \(\ln Z(\alpha, \beta)\) and its derivatives-- one can work with the approximations of  $\Psi_m$.We show that the same statistical guarantees continue to hold, provided the approximation error vanishes faster than \(m^{-1/2}\). The next result formalizes this statement. Proof appear in Appendix~\ref{sec: proof-mle}.

\begin{theorem}[Asymptotic Normality under Approximate Score]\label{thm:approximate mle}
Suppose the rankings \(\pi^{(1)}, \dots, \pi^{(m)}\) are drawn i.i.d.\ from the Mallows distribution \(\mathbb{P}_{\alpha_0, \beta_0, \sigma_0}\) with true parameters \(\alpha_0, \beta_0, \sigma_0\), where \((\alpha_0, \beta_0)\) lies in a compact subset of \( \Theta\in(0, \infty)^2\). Let \((\tilde{\alpha}_m, \tilde{\beta}_m)\) be approximate estimators that solve the estimating equation
\[
\widetilde{\Psi}_m(\tilde{\alpha}_m, \tilde{\beta}_m) := \Psi_m(\tilde{\alpha}_m, \tilde{\beta}_m) + \Delta(\tilde{\alpha}_m, \tilde{\beta}_m) = 0,
\]
where the approximation error \(\Delta(\alpha, \beta)\) satisfies
$\sup_{(\alpha, \beta) \in \Theta} \| \Delta(\alpha, \beta) \| = o(m^{-1/2}).$
 Then, the conclusions of Theorem~\ref{thm:main_mle} continue to hold for \((\tilde{\alpha}_m, \tilde{\beta}_m)\), i.e.,     \[
    \sqrt{m}
    \begin{pmatrix}
        \tilde{\alpha}_m - \alpha_0 \\
        \tilde{\beta}_m - \beta_0
    \end{pmatrix}
    \xrightarrow{d}
    \mathcal{N}\left(
        \mathbf{0},\;
        \mathcal I_{\alpha_0, \beta_0}^{-1}
    \right).
    \]
\end{theorem}

\section{Efficient Sampling}\label{sec:ptas}
Sampling rankings play a central role in many  applications. In sports analytics, it enables the simulation of tournament outcomes; in recommendation systems, it supports generating user suggestions.
Further, sampling is needed for computing the normalizing constant (also known as the partition function) required for maximum likelihood estimation of the Mallows model. 

Exact computation of this partition function becomes infeasible as \(n\) grows—indeed, it is equivalent to Kemeny’s
consensus ranking problem, which is known to be NP-hard~\cite{shah2018simple}.
Although Markov chain Monte Carlo (MCMC) methods have been developed to approximate sampling ~\cite{ Jerrum:2004:PAA:1008731.1008738,  bezakova2008accelerating}
(often phrased in terms of approximating the permanent of a matrix), the MCMC methods rarely scale well in practice~\cite{newman2020fpras,Sankowski}.

Instead of relying on these general-purpose algorithms, we exploit the structure of the  $L_\alpha$-Mallow's model to efficiently sample rankings. 
Our main result on sampling is summarized next:
\begin{theorem}\label{thm:main sampling}
For any \( \alpha \geq 1 \) and \( \beta > 0 \), and any desired accuracy \( \epsilon > 0 \), there exists a fully polynomial-time approximation scheme (FPTAS) that, for the Mallow's distribution, (i)  estimates the partition function \( \Z_n(\beta, \alpha) \) within a multiplicative factor of \( 1 - \epsilon \), and (ii) generates samples $\epsilon$-close (in total variation distance) to the true distribution.
\end{theorem}

\textbf{Proof Outline:}
Throughout the proof, we assume without loss of generality that the central ranking is identity (i.e., $\sigma(i)=i$).
 Our starting point is the construction of a matrix \( {\bA}_n \) whose permanent equals \( Z_n(\beta, \alpha) \). Specifically, let 
${\bA}_n[i, j] = e^{-\beta |i - j|^\alpha},$
for \( i, j \in [n] \).
Note that sampling a permutation according to the probabilities defined by ${\bA}_n$ corresponds to sampling from the $L_\alpha$-Mallow's with dispersion $\beta$. For instance, assigning items  along the diagonal corresponds to the central ranking.

A crucial step in our proof (Lemma~\ref{lm:Pij_decay}) shows that the probability of assigning any item to a position far from its location in the central ranking decays exponentially.
Using this property, we approximate ${\bA}_n$ by truncation -- setting entries sufficiently far from the diagonal to zero. This truncation approach,  explained in detail in Sections~\ref{sec:truncate approx}, yields a probabilistic ranking model within $\epsilon$-total variation distance to the true distribution.  The resulting truncated structure then enables the use of an efficient dynamic programming (DP) algorithm for sequential sampling, which we present in Section~\ref{sec:samplingDP}.

\subsection{Marginal Decay}\label{sec:decay}
We begin by showing that marginal probabilities in the  $L_\alpha$
Mallow's distribution exhibits exponential decay away from the diagonal.
 

\begin{lemma}\label{lm:Pij_decay}
Let \( \alpha \geq 1 \), $\beta>0$. There exist constants \( k > 0 \) and \( c(\alpha, \beta) < 1 \), independent of \( n \), such that for all \( i, j \in [n] \) with \( |j - i| \geq k \),
\[{\P}_{\alpha, \beta}\left(\pi(i) = j \right)\leq c(\alpha,\beta)^{|j-i|}.\]
Furthermore,  $c(\alpha,\beta)$ is monotone decreasing in $\alpha$ and $\beta$, with $c(1,\beta)\leq \frac{2e^{-2\beta}}{1+e^{-2\beta}}$.
\end{lemma}
This is the key lemma used to design the sampling algorithm. The proof carefully analyzes marginal probability ratios  $\P_{\alpha,\beta}(\pi(i)=j+1)/\P_{\alpha,\beta}(\pi(i)=j)$ by comparing permutations assigning $i$ to neighboring positions $j$ vs. $j+1$. Specifically, we define two types of mappings: The first involves simple transpositions that often directly increase permutation probabilities by resolving inversions. When such direct mappings fail to increase probability, we use a second, more complicated mapping that ensures the resulting permutation moves closer to the identity and thus achieves higher probability.
Combining these two mapping types establishes geometric decay in probabilities with respect to  $|j-i|$.
Full proof appears in Appendix~\ref{proof: marginal}.


\subsection{Truncated Distribution}\label{sec:truncate approx}


We approximate $\P_{\alpha,\beta}$ by truncating its distribution. 
Define the  \emph{truncated matrix} \( {\bA}_n^{(k)} \) by zeroing out entries of  \( {\bA}_n \) that lie far from the main diagonal: ${\bA}_n^{(k)}[i,j]=e^{-\beta |i - j|^\alpha}\mathbf{1}_{\{|i-j|\leq k\}}$. 
Let ${\P}_{\alpha,\beta}^{(k)}$ denote the probability distribution induced by this truncated matrix. We show that for \( k = O(\log n)\)  the truncated distribution $ \P_{\alpha,\beta}^{(k)}$ closely approximates the original Mallows distribution.

\begin{center}
    
\begin{tikzpicture}[scale=0.4]

\node at (2.5,5.5) { $\bA_n$};
\foreach \i in {0,...,4} {
  \foreach \j in {0,...,4} {
    \fill[gray!30] (\j,5-\i) rectangle ++(1,-1);
    \draw (\j,5-\i) rectangle ++(1,-1);
  }
}

\node at (9.7,5.6) { $\bA_n^{(k)}$};
\foreach \i in {0,...,4} {
  \foreach \j in {0,...,4} {
    \pgfmathsetmacro{\diff}{abs(\i - \j)}
    \ifdim \diff pt > 1pt
      \fill[white] (7+\j,5-\i) rectangle ++(1,-1);
    \else
      \fill[gray!50] (7+\j,5-\i) rectangle ++(1,-1);
    \fi
    \draw (7+\j,5-\i) rectangle ++(1,-1);
  }
}

\draw[<->, thick, blue] (11.5,.1) -- (11.5,1.9);
\node[black] at (13.5,1) {\small $k$-band};

\end{tikzpicture}
\end{center}

\begin{lemma}\label{lm:tv_distance}
Given \( \epsilon > 0 \), \( \alpha \geq 1 \), and \( \beta > 0 \),  let \( k  = \log(n/\epsilon) \Big(\log\big({\frac{1+e^{-2\beta}}{2e^{-2\beta}}}\big)\Big)^{-1} \), then\\
1. $\|\P_{\alpha,\beta} - {\P}_{\alpha,\beta}^{(k)}\|_{\text{TV}} \leq \epsilon$, where $\text{TV}$ denotes the total variation distance. \\
2. $\text{per}({\bA}_n^{(k)})$ is a multiplicative approximation of $\Z_n(\alpha,\beta)$, i.e., 
$\frac{|\text{per}({\bA}_n) - \text{per}({\bA}_n^{(k)})|}{\text{per}({\bA}_n)} \leq \epsilon.$
\end{lemma}
To prove this result, we use Lemma~\ref{lm:Pij_decay} which shows that the probability of a permutation containing at least one element displaced by more than distance $k$ decays exponentially with $k$ (Lemma~\ref{lm:Pij_decay}). Therefore, if we choose $k=O(\log n)$, a union bound over all indices yields that the total variation distance between ${\P}_{\alpha,\beta}^{(k)}$ and $\P_{\alpha,\beta}$ is sufficiently small. 
Full proof appears in Appendix~\ref{sec: truncation proof}.

\subsection{Sampling from Truncated Distribution}\label{sec:samplingDP}
The final step is to efficiently sample a permutation \(\pi\) from the (truncated) distribution \({\P}_{\alpha,\beta}^{(k)}(\pi)\). 
We achieve this via dynamic programming (DP), structured in layers  that assign elements to positions.
Specifically, a DP state \(DP[i][s]\) tracks the cumulative weight of all partial permutations that have assigned the first \(i-1\) elements, where the state \(s\) indicates available columns. Recall that each element $i$ sampled from  ${\P}_{\alpha,\beta}^{(k)}(\pi)$ can be only matched to those \(\pm k\) around \(i\). Thus, the state \(s\) is a binary vector of length \(2k\), marking columns as available (0) or assigned (1). Transitions between layers assign the current element \(i\) to an available column, updating the state to mark this column as assigned, and weighting the transition by the corresponding entry from \({\bA}_n^{(k)}\) (see Figure~\ref{fig:dp}). 
The final DP state at layer  captures the total weight of all perfect matchings, equal to the permanent of  \({\bA}_n^{(k)}\).

Once the forward DP pass is complete, 
sampling can be done by reversing through DP states, choosing transitions with probability proportional to their DP-computed weights (marginal probabilities).
Detailed algorithms and rigorous analysis for DP updates, graph construction, and sampling appear in Appendices \ref{sec:dp construct} and \ref{sec:dp proof} (Algorithms \ref{alg:build_tag} and \ref{alg:sample_dp}).



\begin{lemma}\label{thm:DP}
For any given $n$ and $k$: Algorithm~\ref{alg:build_tag} constructs the DP table and $\text{per}({\bA}_n^{(k)})$ in time \( O\left( n\binom{2k}{k} \right).\) Then given the precomputed DP table, Algorithm~\ref{alg:sample_dp} samples each permutation from $\hat{\mu}_n^{(k)}$ in time \( O\left( nk \right) \).  
\end{lemma}
Combining Lemma~\ref{thm:DP} with the truncation guarantee in Section~\ref{sec:truncate approx}, we obtain the PTAS described in Theorem~\ref{thm:main sampling}. Specifically, for $k=O(\log(n/\epsilon))$   the total variation distance between the true Mallow's distribution and the truncated model is at most $\epsilon$.
Under this setting, the preprocessing step takes time $O(\epsilon^{-C}n^{1+C})$, where $C = 2 \log((1+e^{-2\beta})/2e^{-2\beta})$ (constant derived from Lemma~\ref{lm:tv_distance}). After this setup, each sample can be drawn in $O(n\log(n/\epsilon))$ time. Together, these results yield a PTAS for both partition function estimation and sampling.
\begin{figure}
    \centering
    \includegraphics[width=0.7\linewidth]{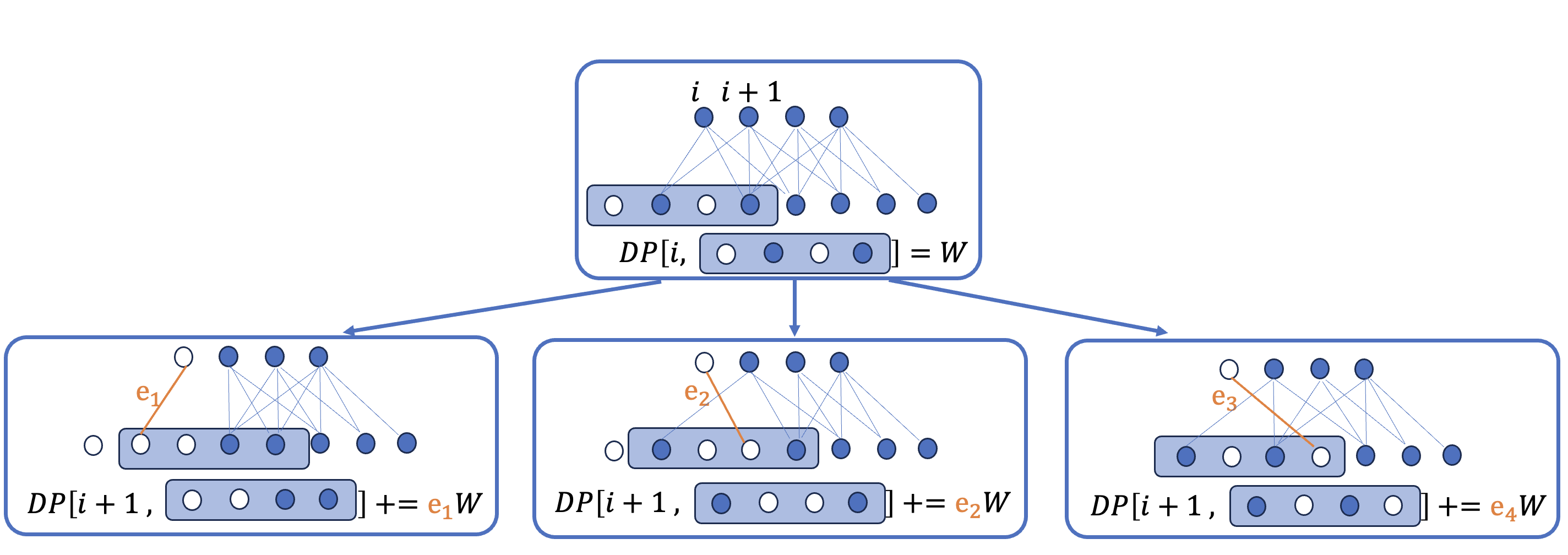}
    \caption{DP state transitions showing available (white) and assigned (blue) columns within bandwidth ($k=2$). Matched edges for $i$ are shown in orange; DP values represent weights of partial matchings.}
    \label{fig:dp}
\end{figure}

\section{Experiments}\label{sec:empiric}

We empirically evaluate our methods with two primary goals:  (1) assessing the predictive performance of our model compared to classical ranking models on real-world and synthetic datasets; and (2) validating our approximation scheme for the partition function \(Z_n(\beta,\alpha)\) and the sampling algorithm.

\textbf{Reproducibility and Computational Resources.}
Our implementation, which also includes tools for sampling and learning from $L_\alpha$ Mallows models, is available at \cite{asgari2025code}. The MLE evaluation experiments were conducted on a server with dual AMD EPYC 7763 64-Core processors and 1.5 TB of RAM. Sampling experiments were run on an Intel Core i7 CPU with 16 GB RAM.


\subsection{Predictive Performance of the {$\boldsymbol{L}_{\boldsymbol{\alpha}}$-}Mallow's Model}\label{sec:experimetins_MLE}


\paragraph{Datasets.} We evaluate our methods on the following datasets:
\\ \textbf{American college football rankings:}  
    We use weekly power rankings of college football teams from the 2019 to 2021 seasons, sourced from \cite{masseyratings_rankings}, providing over 4,100 distinct observed rankings across various expert sources.  These rankings—known as \emph{power rankings}—are ordinal lists produced by independent media outlets and analysts. Each ranking reflects an expert's assessment of team strength during a particular week, typically incorporating factors such as win–loss records, strength of schedule, point differentials, and recent performance. Unlike standings, which are derived purely from match outcomes, power rankings aim to capture an underlying strength signal.
    We focus on learning the top 10 teams identified in \cite{bleacherreport2022apcfp}: Georgia, Alabama, Michigan, Cincinnati, Baylor, Ohio State, Oklahoma State, Notre Dame, Michigan State, and Oklahoma. To evaluate scalability, we also conduct experiments on a larger set of the top 100 most frequently ranked teams.
    \\ \textbf{College basketball rankings:}  
  We similarly use weekly power rankings from the 2020 and 2021 college basketball seasons, also sourced from \cite{masseyratings_rankings}, resulting in over 1,700 observed rankings. 
    We further, select the top 10 teams as reported in \cite{collegepollarchive2021ap}: Gonzaga, Illinois, Baylor, Michigan, Alabama, Houston, Ohio State, Iowa, Texas, and Arkansas. For large-scale evaluation, we additionally sample a set of 100 teams at random.
    \\ \textbf{\textsc{Sushi} preference rankings:}  
    We also include the \textsc{Sushi} preference dataset from \cite{sushi2005dataset}, which contains 5,000 human-annotated rankings over 10 different types of \textsc{Sushi}.\\\textit{Synthetic data:} We generate rankings using our sampler (Algorithm~\ref{alg:sample_dp}). The parameters are set to $\alpha = 1.5$, $\beta = 0.5$, with $n = 15$ items, and the truncation of  $k = 9$ is used for sample generation.

\paragraph{Data Splitting Strategy.}
To avoid temporal leakage in the sports datasets, we train on earlier rankings and hold out the final sixth of the data for evaluation. We repeat this sampling procedure independently over 50 trials to derive confidence intervals. In each of 50 trials, we randomly sample 700 training and 150 testing examples for basketball, and 800 training and 250 testing examples for football.
For the \textsc{Sushi} dataset, we conduct 25 independent trials, each with a random  $80\%$- $20\%$  training-testing split. We report mean and standard deviation across trials.

\paragraph{Baselines.}
To highlight the performance of our model, we compare against two classical models:
    \textit{Plackett--Luce (PL) model~\cite{plackett1975analysis}:} 
    The PL model is a classical ranking model that assumes rankings are generated sequentially. At each stage, the probability of choosing item~$i$ for the next available position is proportional to a parameter~$\theta_i$. Once selected, the item is removed from the pool, and the process continues with the remaining items. Formally,
    \[
    \mathbb{P}_{\theta}^{\mathrm{PL}}(\pi) = \frac{\prod_{i=1}^{n} \theta_{\pi(i)}}{\prod_{i=1}^{n} \left( \sum_{j=i}^n \theta_{\pi(j)} \right)}.
    \]

    \textit{Mallow's $\tau$ (Kendall's tau distance):} 
    Mallow's $\tau$ assumes the underlying distance metric is
    the Kendall's tau distance~$d_{\tau}(\pi, \sigma)$ \cite{kendall1938new}, which counts the number of pairwise disagreements (inversions) between $\pi$ and $\sigma$:
    \[
    d_{\tau}(\pi, \sigma) = 
    \left| \left\{ (i, j) : i < j,\; (\pi(i) < \pi(j)) \ne (\sigma(i) < \sigma(j)) \right\} \right|.
    \]
We defer the details of implementation to the Appendix~\ref{exp-app}.

\paragraph{Metrics.}
We measure predictive accuracy via the following metrics: (1) \textit{Top-1/top-5 hit rates}: Probability the test ranking's top item appears in the top-1/top-5 predicted ranks  (2) \textit{Spearman's $\rho$}: rank correlation based on squared rank differences; (3) \textit{Kendall's $\tau$}: correlation based on pairwise agreements; (4) \textit{Hamming distance}: fraction of differing positions; (5) \textit{Pairwise accuracy}: agreement ratio of item pairs' relative ordering. 
Arrows next to each metric indicate whether higher ($\uparrow$) or lower ($\downarrow$) values correspond to better performance.
Full description of these metrics appears in Appendix~\ref{exp-app}.


\paragraph{MLE Evaluation on Real-World Datasets.}
Tables~\ref{tab:mle_comparison_football},~\ref{tab:mle_comparison_basketball}, and~\ref{tab:mle_comparison_sushi} present results for football, basketball, and \textsc{Sushi} datasets, respectively. The best results are bold-faced, showing that the \(L_\alpha\)-Mallows model consistently outperforms the Plackett-Luce model and Mallow's $\tau$ model over several metrics on all datasets.
Notably, the basketball dataset results (Table~\ref{tab:mle_comparison_basketball}) highlight particularly poor performance of the Plackett–Luce model, indicated by negative correlation values, possibly due to overfitting. The better generalization of the $L_\alpha$ Mallows model is reflected in significantly higher predictive accuracy. 


The parameters estimated from the \(L_\alpha\)-Mallows model offer meaningful insights into ranking behavior. The dispersion parameter (\(\beta\)) captures stability—lower values imply frequent ranking shifts, while higher values indicate stability. The learned distance parameter quantifies ranking dynamics, highlighting deviations of the optimal model from traditional models like Kendall's $\tau$ and Spearman's $\rho$ (e.g. $\alpha=2$).
The higher estimated $\alpha \approx 1.1$ indicates a lower likelihood of long-range swaps (e.g., a team ranked 10th defeating a top-ranked team) compared to football, which has a smaller $\alpha \approx 0.4$, allowing more frequent long-range swaps.

\begin{table}[h!]
\centering
\setlength{\tabcolsep}{8pt}
\begin{tabular}{|l|c|c|c|}
\hline
 & \textbf{$\boldsymbol{L}_{\boldsymbol{\alpha}}$-Mallow} & {Mallow's $\tau$} & {Plackett--Luce} \\
\hline
Estimated $\alpha$ & 0.442 \;($\pm$ 0.061) & -- & -- \\
Estimated $\beta$  & 0.455 \;($\pm$ 0.051) & -- & -- \\
\hline
$\uparrow$Spearman's $\rho$ correlation      & \textbf{0.094} \;($\pm$ 0.010) & 0.070 \;($\pm$ 0.016) & 0.093 \;($\pm$ 0.009) \\
$\uparrow$Kendall's $\tau$ correlation       & \textbf{0.070} \;($\pm$ 0.007) & 0.052 \;($\pm$ 0.011) & 0.065 \;($\pm$ 0.007) \\
$\downarrow$Hamming distance                 & \textbf{0.888} \;($\pm$ 0.005) & 0.892 \;($\pm$ 0.002) & 0.920 \;($\pm$ 0.001) \\
$\uparrow$Pairwise accuracy (\%)             & \textbf{53.5} \;($\pm$ 0.4)    & 52.6 \;($\pm$ 0.6)    & 53.3 \;($\pm$ 0.3)    \\
$\uparrow$Top-1 hit rate (\%)                & \textbf{8.0}  \;($\pm$ 0.1)    & 6.9  \;($\pm$ 0.7)    & 3.5  \;($\pm$ 0.1)    \\
$\uparrow$Top-5 hit rate (\%)                & \textbf{41.9} \;($\pm$ 0.5)    & 41.1 \;($\pm$ 2.9)    & 30.5 \;($\pm$ 0.7)    \\
\hline
\end{tabular}
\caption{College football dataset, model out-of-sample 
 performance averaged over 50 independent trials (mean $\pm$ standard deviation). }
\label{tab:mle_comparison_football}
\end{table}

\begin{table}[h!]
\centering
\setlength{\tabcolsep}{8pt}
\begin{tabular}{|l|c|c|c|}
\hline
 & \textbf{$\boldsymbol{L}_{\boldsymbol{\alpha}}$-Mallow} &  {Mallow's $\tau$} & {Plackett--Luce} \\
\hline
Estimated $\alpha$ & 1.096 \;($\pm$ 0.056) & -- & -- \\
Estimated $\beta$  & 0.178 \;($\pm$ 0.019) & -- & -- \\
\hline
$\uparrow$Spearman's $\rho$ correlation   & \textbf{0.269} \;($\pm$ 0.005) & 0.174 \;($\pm$ 0.011) & $-0.020$ \;($\pm$ 0.009) \\
$\uparrow$Kendall's $\tau$ correlation    & \textbf{0.199} \;($\pm$ 0.004) & 0.128 \;($\pm$ 0.009) & $-0.010$ \;($\pm$ 0.007) \\
$\downarrow$Hamming distance              & \textbf{0.872} \;($\pm$ 0.001) & 0.880 \;($\pm$ 0.001) & 0.919 \;($\pm$ 0.001) \\
$\uparrow$Pairwise accuracy (\%)          & \textbf{59.9} \;($\pm$ 0.2)    & 56.4 \;($\pm$ 0.4)    & 49.5 \;($\pm$ 0.3)    \\
$\uparrow$Top-1 hit rate (\%)             & \textbf{22.7} \;($\pm$ 0.6)    & 17.5 \;($\pm$ 0.3)    & 4.9  \;($\pm$ 0.3)    \\
$\uparrow$Top-5 hit rate (\%)             & \textbf{77.7} \;($\pm$ 1.0)    & 68.4 \;($\pm$ 0.5)    & 57.7 \;($\pm$ 1.6)    \\
\hline
\end{tabular}
\caption{College basketball dataset, model out-of-sample performance averaged over 50 independent trials (mean $\pm$ standard deviation).}
\label{tab:mle_comparison_basketball}
\end{table}

\begin{table}[h!]
\centering
\setlength{\tabcolsep}{8pt}
\begin{tabular}{|l|c|c|c|}
\hline
 & \textbf{$\bold{L}_{\boldsymbol{\alpha}}$-Mallow} &  {Mallow's $\tau$} & {Plackett-Luce} \\
\hline
Estimated $\alpha$ & 0.764 \;($\pm$ 0.378) & -- & -- \\
Estimated $\beta$  & 0.159 \;($\pm$ 0.172) & -- & -- \\
\hline
$\uparrow$ Spearman's $\rho$ correlation   & \textbf{0.043} \;($\pm$ 0.006) & 0.031 \;($\pm$ 0.012) & -0.050 \;($\pm$ 0.003) \\
$\uparrow $Kendall's $\tau$ correlation    & \textbf{0.031} \;($\pm$ 0.005) & 0.023 \;($\pm$ 0.009) & -0.037 \;($\pm$ 0.002) \\
$\downarrow$Hamming distance  & \textbf{0.893} \;($\pm$ 0.001) & {0.897} \;($\pm$ 0.001) & 0.916 \;($\pm$ 0.000) \\
$\uparrow $Pairwise accuracy (\%)          & \textbf{51.6} \;($\pm$ 0.2)    & 51.1 \;($\pm$ 0.4)    & {48.2} \;($\pm$ 0.1)    \\
$\uparrow$Top-1 hit rate (\%)            & 10.2 \;($\pm$ 0.6)    & \textbf{10.6} \;($\pm$ 0.6)    & {9.3} \;($\pm$ 0.1)     \\
$ \uparrow$Top-5 hit rate (\%)            & \textbf{52.0} \;($\pm$ 1.0)    & 
51.7 \;($\pm$ 1.6)    & {49.5} \;($\pm$ 0.3)    \\
\hline
\end{tabular}
\caption{\textsc{Sushi} dataset, model out-of-sample performance averaged over 25 independent trials (mean $\pm$ standard deviation).}
\label{tab:mle_comparison_sushi}
\end{table}

\textbf{Large $\boldsymbol{n}$ Regime: MLE Evaluation on 100-Team College Sports Rankings.}
Our learning algorithm  can efficiently scale to rankings over a large number of items (see the theoretical guarantees in \ref{thm:main sampling}). To evaluate the robustness and scalability of our model, we compared the performance of the $L_\alpha$ Mallow model against Kendall’s $\tau$-based model and Placket–Luce model on college sports datasets, this time considering $100$ teams.

For the football dataset, teams were selected with the highest participation rates to ensure a sufficient number of full rankings remained. For the basketball dataset, $100$ teams were selected at random. The results for both datasets are reported in Table~\ref{tab:mle_comparison_basketball_100_teams} and Table~\ref{tab:mle_comparison_football_100_teams}, respectively. All experiments were conducted with truncation size $k=7$.

Tables~\ref{tab:mle_comparison_basketball_100_teams} and \ref{tab:mle_comparison_football_100_teams} report out-of-sample performance averaged over 50 independent trials. Across both datasets, the $L_\alpha$-Mallows model significantly outperforms classical baselines (Mallows-$\tau$, Plackett--Luce) across all metrics. The flexibility of choosing the distance metric is particularly evident as the number of teams increases in both datasets, the $L_\alpha$ Mallow model achieves more than fivefold improvement in top-1 hit rate compared to the baselines, significantly outperforming both Kendall’s $\tau$ and Placket–Luce models.

\begin{table}[h!]
\centering
\setlength{\tabcolsep}{8pt}
\begin{tabular}{|l|c|c|c|}
\hline
 & \textbf{$\boldsymbol{L}_{\boldsymbol{\alpha}}$-Mallow} &  {Mallow's $\tau$} & {Plackett--Luce} \\
\hline
 Estimated $\alpha$      & 0.002 ($\pm$ 0.002)  & --               & --               \\
 Estimated $\beta$       & 0.504 ($\pm$ 0.023)  & --               & --  
 \\
 \hline
$\uparrow$Spearman's $\rho$ correlation    & \textbf{0.759} ($\pm$ 0.006)  & 0.478 ($\pm$ 0.006)  & 0.373 ($\pm$ 0.006)  \\
$\uparrow$Kendall's $\tau$ correlation    & \textbf{0.564} ($\pm$ 0.005)  & 0.328 ($\pm$ 0.005)  & 0.253 ($\pm$ 0.004)  \\
$\downarrow$Hamming distance     & \textbf{0.973} ($\pm$ 0.000)  & 0.984 ($\pm$ 0.000)  & 0.986 ($\pm$ 0.000)  \\
$\uparrow$Pairwise accuracy (\%)   & \textbf{96.378 }($\pm$ 0.100) & 91.754 ($\pm$ 0.152) & 90.298 ($\pm$ 0.138) \\
$\uparrow$Top-1 hit rate (\%)     & \textbf{12.873} ($\pm$ 0.941) & 2.212 ($\pm$ 0.807)  & 1.749 ($\pm$ 0.250)  \\
$\uparrow$Top-5 hit rate (\%)      & \textbf{56.465 }($\pm$ 2.315) & 10.857 ($\pm$ 1.401) & 7.205 ($\pm$ 0.605)  \\
\hline
\end{tabular}
\caption{College basketball dataset for 100 teams, model out-of-sample performance averaged over 50 independent trials (mean $\pm$ standard deviation).}
\label{tab:mle_comparison_basketball_100_teams}
\end{table}

\begin{table}[h!]
\centering
\setlength{\tabcolsep}{8pt}
\begin{tabular}{|l|c|c|c|}
\hline
 & \textbf{$\boldsymbol{L}_{\boldsymbol{\alpha}}$-Mallow} &  {Mallow's $\tau$} & {Plackett--Luce} \\
\hline
 Estimated $\alpha$      & 0.003 ($\pm$ 0.002)  & --               & --               \\
 Estimated $\beta$       & 0.516 ($\pm$ 0.030)  & --               & --               \\
 \hline
$\uparrow$Spearman's $\rho$ correlation      & \textbf{0.454 }($\pm$ 0.006)  & 0.387 ($\pm$ 0.007)  & 0.138 ($\pm$ 0.005)  \\
$\uparrow$Kendall's $\tau$ correlation  & \textbf{0.318} ($\pm$ 0.004)  & 0.264 ($\pm$ 0.005)  & 0.093 ($\pm$ 0.004)  \\
$\downarrow$Hamming distance    & \textbf{0.981} ($\pm$ 0.000)  & 0.986 ($\pm$ 0.000)  & 0.989 ($\pm$ 0.000)  \\
$\uparrow$Pairwise accuracy (\%)  &\textbf{ 91.163} ($\pm$ 0.165) & 88.134 ($\pm$ 0.178) & 86.944 ($\pm$ 0.133) \\
$\uparrow$Top-1 hit rate (\%) & \textbf{2.057} ($\pm$ 0.535)  & 0.294 ($\pm$ 0.399)  & 1.524 ($\pm$ 0.899)  \\
$\uparrow$Top-5 hit rate (\%) & \textbf{23.590} ($\pm$ 3.979) & 2.386 ($\pm$ 1.240)  & 7.097 ($\pm$ 1.154)  \\
\hline
\end{tabular}
\caption{College football dataset for 100 teams, model out-of-sample performance averaged over 50 independent trials (mean $\pm$ standard deviation).}
\label{tab:mle_comparison_football_100_teams}
\end{table}


Notably, the estimated distance parameter ($\alpha$) remains close to zero in both datasets, suggesting that the model adaptively flattens the penalty for long-range swaps when necessary, capturing the heterogeneity in upset dynamics across sports. This reinforces the utility of learning $\alpha$ from data rather than fixing it a priori.


\textbf{Effect of the Truncation size.}\\
\noindent
We demonstrate the exponential decay of the $\ell_2$ estimation error on synthetic data. Specifically, we generate $50$ training samples over $15$ items using our sampling method with truncation order $k = 9$ and parameters $\alpha_0 = 1$ and $\beta_0 = 1$. We then train our model using varying truncation levels and report the mean and standard deviation of the estimation error in Figure~\ref{fig:estimation_error_vs_k}. Notably, even though the training data was generated with a truncation of order $9$, using a truncation as small as $k = 5$ yields stable and accurate estimates, highlighting the robustness of the model.

\begin{figure}[h]
    \centering
    \includegraphics[width=0.45\linewidth]{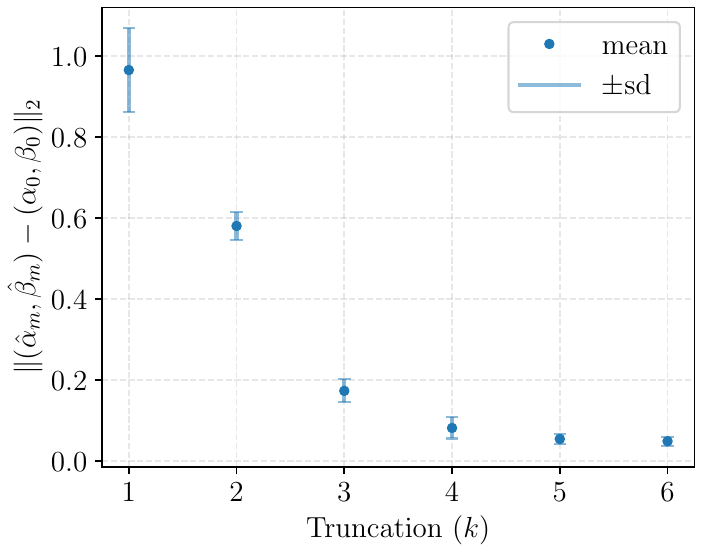}
    \caption{
        Estimation error versus the truncation size used for training $L_\alpha$ Mallow's model.
        Each point shows the $\ell_2$-test error, averaged across 25 independent trials.
        As $k$ increases, the model receives exponentially more information, leading to an exponential decay rate of the estimation error.
    }
    \label{fig:estimation_error_vs_k}
\end{figure}

\paragraph{Synthetic Validation.} 
We further establish the accuracy and robustness of MLE on the synthetic dataset.  Specifically, we train our model using samples generated by the truncated sampling algorithm (Algorithm~\ref{alg:sample_dp}) with a smaller truncation parameter ($k=6$) compared to the truncation parameter used to generate the original synthetic data ($k=9$). Despite this deliberate mismatch, Figure~\ref{fig:estimation_error} illustrates that our MLE procedure accurately recovers the true underlying parameters $\alpha$ and $\beta$, highlighting the robustness of our estimation method to the truncation choice.


\begin{figure}[h!] 
\centering
\includegraphics[width=0.4\textwidth]{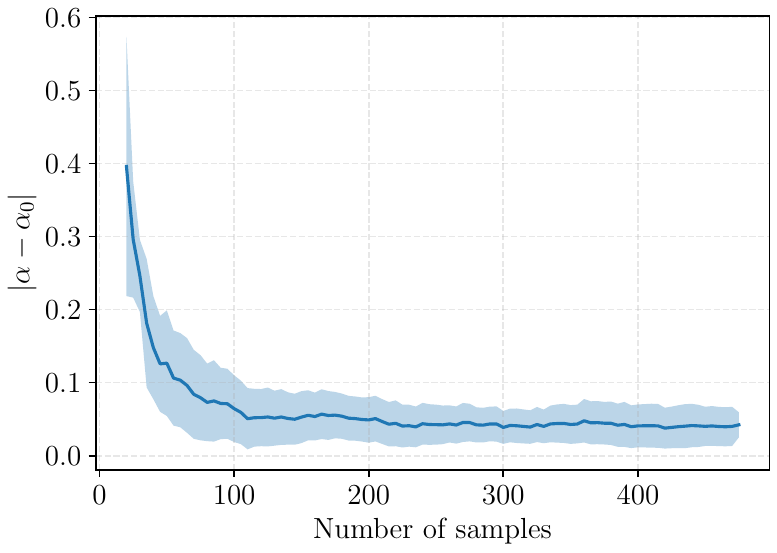}%
\hspace{0.01\textwidth}%
\includegraphics[width=0.4\textwidth]{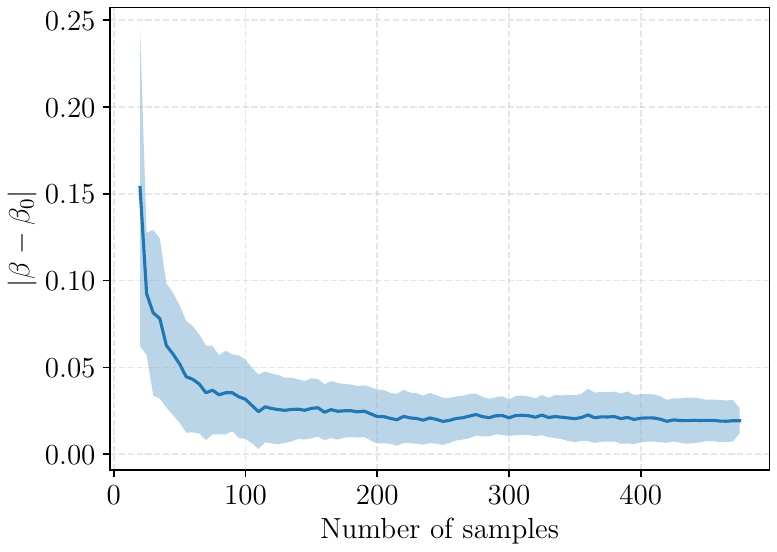}
\caption{The estimation error of $\alpha$ (left) and $\beta$ (right), for different training sizes over the synthetic data with $\alpha_0=1.5$ and $\beta_0=0.5$. The shaded area represents the standard deviation, computed over $25$ independent trials. Our results show negligible estimation bias, despite using  the truncation algorithm.}
\label{fig:estimation_error}
\end{figure}
The synthetic experiments, where the ground truth parameters are known, confirm the accuracy and robustness of our parameter recovery.

Overall, our experiments validate that the  \(L_\alpha\)-Mallows model provides superior predictive accuracy, meaningful parameter interpretability, and computational efficiency, positioning it as an attractive method for ranking analysis across diverse datasets.

\subsection{Validation of the Sampling Algorithm}
\label{sec:experiment_sampling}
\textbf{Partition function.} We first validate the accuracy of our approximation for the partition function \(Z_n(\alpha, \beta)\). We quantify approximation quality by measuring the relative error:
   $ \frac{|Z_n(\alpha, \beta) - \hat{Z}_n(\alpha, \beta)|}{Z_n(\alpha, \beta)},$
across different values of  \(n\), with a fixed truncation parameter \(k=3\). 
Additionally, we report the efficiency of our estimation by comparing our approximation's runtime against the exact computation, which computes the partition function constant exactly. We use Ryser's algorithm~\cite{ryser1963combinatorial}
, which is regarded as one of the most efficient permanent computation algorithms. Despite using Ryser's algorithm, we still observe the overwhelming difference in the efficiency of our approximation, while only enduring negligible error. 
Table~\ref{tab:permanent_error} summarizes our findings. We consistently observe small relative errors (on the order of \(10^{-4}\)), validating the high accuracy of our approximation. More importantly, our approximation significantly reduces computation time, especially as \(n\) increases.
Figure~\ref{fig:runtime_comparison} provides a visual comparison of runtime ratios and relative errors across different \(\alpha\) and \(\beta\). 
\begin{table}[h!]
    \centering
    \setlength{\tabcolsep}{7pt}  
    \begin{tabular}{{|c|c|c|c|c|}}
        \hline
        \(n\) &  Relative error & Algorithm~\ref{alg:build_tag} & Exact \texttt{computation}  & Time ratio \\
        & & time (s) & time (s) & (Exact/Alg.~\ref{alg:build_tag} ) \\
        \hline
        6  & $(1.25 \pm 0.03)\times 10^{-4}$ & $0.011 \; (\pm 0.002)$ & $0.355 \; (\pm 0.012)$ & $\textbf{32.0} \;
        (\pm \textbf{5.9}$) \\ 
        8  & $(2.10 \pm 0.05)\times 10^{-4}$ & $0.027 \; (\pm 0.003)$ & $32.25 \;(\pm 0.98)$ & $\textbf{1194.5} \;
        (\pm \textbf{135.8})$ \\
        10 & $(2.94 \pm 0.06)\times 10^{-4}$ & $0.057 \; (\pm 0.004)$ & $6471.3 \;(\pm 220.25)$ & $\textbf{113531.6} \;(\pm \textbf{8166.4})$ \\
        \hline
    \end{tabular}
    \caption{Relative error and runtime comparison between our approximation method (truncation order \(k = 3\)) and exact partition function computation via Ryser's algorithm, with parameters \(\alpha = 1\) and \(\beta = 2\). Results are averaged over 1000 trials, and 95\% confidence intervals are reported.}
    \label{tab:permanent_error}
\end{table}

\begin{figure}[h!]
    \centering
    \includegraphics[width=0.9\textwidth]{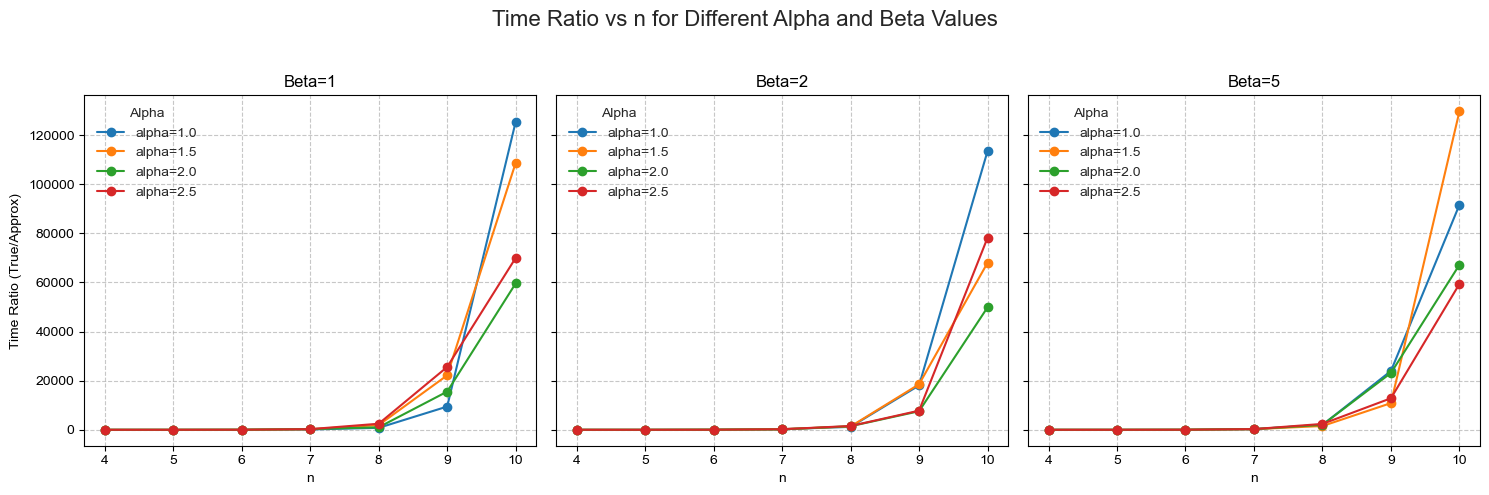}
    \caption{Runtime comparison between the permanent approximation (\(k=3\)) and the full computation for varying values of \(n\).}
    \label{fig:runtime_comparison}
\end{figure}

\paragraph{Limitations and Future Work.}
Our results highlight several promising directions, yet they also point clearly to opportunities for future work. A fundamental theoretical limitation lies in our efficient sampling algorithm, which currently requires a  (\(\alpha \geq 1\)). Developing computationally efficient sampling for the broader class of distance functions with \(\alpha < 1\) remains open. On the practical side, extending the empirical scope beyond sports rankings—potentially examining applications in recommendation systems, voting scenarios, or other domains—represents an exciting and natural next step, allowing further validation of the generalized \(L_\alpha\)-Mallows framework.

\section*{Acknowledgment}
We thank Persi Diaconis for introducing us to the Mallows model and the related literature on ranking. We are also grateful to Amin Saberi for helpful discussions and guidance throughout the development of this work.
\newpage
\bibliographystyle{plain} 
\bibliography{ref} 

\newpage
\appendix

\section{Proofs for Computing Central Ranking Estimator}\label{sec: proof-matching}
We start by proving that normalizing constant $Z_n$ is independent of choice of $\sigma$.
\begin{proposition}[Invariance of Partition Function]\label{lem:Zn-indep-sigma}
For any fixed parameters \(\alpha > 0\) and \(\beta > 0\), the partition function 
\[
Z_n(\beta,\alpha) = \sum_{\pi \in S_n} e^{-\beta d_\alpha(\pi,\sigma)}
\]
is independent of the central ranking \(\sigma \in S_n\).
\end{proposition}
\begin{proof}
Consider any two permutations \(\sigma, \sigma' \in S_n\). 
Define the bijective mapping \(f_{\sigma,\sigma'}:S_n\to S_n\) by
\[
f_{\sigma,\sigma'}(\pi) = \pi \circ \sigma^{-1}\circ\sigma',
\]
which satisfies \(f_{\sigma,\sigma'}(\sigma) = \sigma'\). Since \(f_{\sigma,\sigma'}\) is a bijection, suming over permutations is invariant under re-indexing. Thus,
\[
Z_n(\beta,\alpha,\sigma') 
= \sum_{\pi \in S_n} e^{-\beta d_\alpha(f_{\sigma,\sigma'}(\pi),\sigma')}
= \sum_{\pi \in S_n} e^{-\beta d_\alpha(\pi\circ \sigma^{-1}\circ\sigma',\sigma')}.
\]

Now, observe by definition of the distance \(d_\alpha\) (which depends only on relative positions of elements), we have
\[
d_\alpha(\pi\circ \sigma^{-1}\circ\sigma',\sigma') = d_\alpha(\pi,\sigma).
\]

Hence,
\[
Z_n(\beta,\alpha,\sigma') 
= \sum_{\pi \in S_n} e^{-\beta d_\alpha(\pi,\sigma)}
= Z_n(\beta,\alpha,\sigma).
\]

Since this equality holds for any pair of permutations \(\sigma,\sigma'\), the partition function is invariant and thus independent of the choice of the central ranking.
\end{proof}
Next we bring the proof of Proposition~\ref{prop:update-sigma-matching}, which reduces the optimization problem regarding $\sigma$ to minimum matching problem.
\begin{proof}[Proof of Proposition~\ref{prop:update-sigma-matching}]
Consider a complete bipartite graph \(G=(U\cup V, E)\) with 
\(\lvert U\rvert=\lvert V\rvert=n\).  Label the vertices in \(U\) by \(u_1,u_2,\dots,u_n\)
and those in \(V\) by \(v_1,v_2,\dots,v_n\).  
For each edge \(e=(u_i,v_j)\), assign the weight
\[
w(i,j)
  \;=\;
  -\sum_{l=1}^m
    \bigl|\pi^{(l)}(i)\;-\;j\bigr|^\alpha.
\]

Any perfect matching \(M\subseteq U\times V\) corresponds to a unique permutation 
\(\sigma \in S_n\) via the rule \(\sigma(i)=j\) whenever the edge 
\((u_i,v_j)\in M\).  
Indeed, \(M\) must match every \(u_i\) to exactly one \(v_j\), and thus
\(\sigma\) is well-defined and bijective on \(\{1,\dots,n\}\).

Under this correspondence, the matching \(M\) has total weight
\(\sum_{(u_i,v_j)\in M} w(i,j)\), which is exactly
\[
\sum_{i=1}^n 
  \sum_{l=1}^m
    \bigl|\pi^{(l)}(i)\;-\;\sigma(i)\bigr|^\alpha
  \;=\;
  -\sum_{l=1}^m
    d_\alpha\bigl(\pi^{(l)},\sigma\bigr).
\]
Hence finding optimal $\sigma$
is equivalent to finding a minimum weight perfect matching. 
\end{proof}

\section{Proof of Theorem~\ref{thm:main_mle}}\label{sec: proof-mle}

\subsection{Consistency of $\hat\sigma_m$ (Proof of Lemma~\ref{lm:unqiue minimizer})}
We start with the first part of the theorem which is consistency of $\hat \sigma_m$.
Recall that $\hat\sigma_m$ is answer to the following optimiztion
\[
\hat{\sigma}_m = \arg\min_{\sigma \in S_n} \frac{1}{m}\sum_{l=1}^m d_\alpha(\pi^{(l)}, \sigma).
\]
Note that as the sample size grows, by the law of large number the right hand side converges to $\E_{\pi\sim \P_{\alpha_0,\beta_0,\sigma_0}}[d_{\alpha_0}(\pi, \sigma)].$  We start the proof by showing that in the limit, we can change the distance function inside, and for any choice of $\alpha> 0$ the unique minimizer of this is $\sigma_0$.  

So, we start by proving the second part of Lemma~\ref{lm:unqiue minimizer}, i.e., our goal is to show that given any $\tilde\alpha>0$, the following holds
\[
   \arg\min_{\sigma'\in S_n}\;\E_{\Pi\sim \P_{\alpha_0,\beta_0,\sigma_0}}\!\bigl[d_{\tilde\alpha}(\Pi,\sigma')\bigr]
   \;=\;\{\sigma_0\}.
\]

\begin{proof}
Define
\[
F(\mathbf x)\;=\;\sum_{\pi\in S_n}\Bigl(\sum_{u=1}^n|\pi(u)-x_u|^{\tilde\alpha}\Bigr)\,\P_{\alpha_0,\beta_0,\sigma_0}(\pi)
\;=\;\sum_{u=1}^n\mathbb{E}\bigl[\,|\Pi(u)-x_u|^{\tilde\alpha}\,\bigr],
\]
where  $\Pi$ is sampled from $\P_{\alpha_0,\beta_0,\sigma_0}$.
Notice that if we restrict $x$ to permutations in $S_n$, then we recover the original problem. 
We show that whenever $x\neq\sigma_0$, performing a sequence of transpositions that brings $\mathbf x$ closer to $\sigma_0$ strictly decreases $F(\mathbf x)$, and hence the unique minimizer is $\mathbf x=\sigma_0$.

For simplicity, assume $\sigma_0$ is the identity permutation, and assume there is another $\sigma'$ minimizing $F$. 
We will prove uniqueness of the minimizer in two stages:
(i) First we show, swapping any inverted pair of positions strictly decreases the objective.
(ii) In the second step, by repeatedly correcting inversions we decrease $F(x)$ down to $F(\sigma_0)$.

\noindent\textbf{Step 1: Inversion swap.}  Suppose $\mathbf x=(x_1,\dots,x_n)\in\R^n$ and there is an inverted pair of positions 
$k<i$ with 
$x_k>x_i$.
Let $\mathbf x'=(x'_1,\dots,x'_n)$ be the vector obtained by swapping $x_k$ and $x_i$:
\[
x'_k = x_i,\quad x'_i = x_k,\quad x'_u = x_u\ \text{for}\ u\notin\{k,i\}.
\]
Then  we claim that
\[
F(\mathbf x') \;<\; F(\mathbf x).
\]

Set  $s \;=\; x_i$,  $t \;=\; x_k$
so that $s<t$.  Observe that only the $k$th and $i$th terms in $F$ change under the swap.  Indeed,
\[
F(\mathbf x)-F(\mathbf x')
= \mathbb{E}\bigl[|\Pi(k) - t|^{\tilde\alpha}\bigr] + \mathbb{E}\bigl[|\Pi(i) - s|^{\tilde\alpha}\bigr]
  \;-\;\Bigl(\mathbb{E}\bigl[|\Pi(k) - s|^{\tilde\alpha}\bigr] + \mathbb{E}\bigl[|\Pi(i) - t|^{\tilde\alpha}\bigr]\Bigr).
\]
Rearrange:
\[
F(\mathbf x)-F(\mathbf x')
= \Bigl(\mathbb{E}[|\Pi(i) - s|^{\tilde\alpha}]-\mathbb{E}[|\Pi(i) - t|^{\tilde\alpha}]\Bigr)
  \;-\;\Bigl(\mathbb{E}[|\Pi(k) - s|^{\tilde\alpha}]-\mathbb{E}[|\Pi(k) - t|^{\tilde\alpha}]\Bigr).
\]
Define 
\[
D(j)=|j-s|^{\tilde\alpha}-|j-t|^{\tilde\alpha},\quad j=1,\dots,n.
\]
Since ${\tilde\alpha}>0$, $x\mapsto|x-s|^{\tilde\alpha}-|x-t|^{\tilde\alpha}$ is strictly increasing when $x\in[s,t]$, so
\[
\Delta_j
\;=\;
D(j)-D(j-1)
\;=\;
\bigl|j-s\bigr|^{\tilde\alpha}-\bigl|j-1-s\bigr|^{\tilde\alpha}
\;-\;\bigl|j-t\bigr|^{\tilde\alpha}+\bigl|j-1-t\bigr|^{\tilde\alpha}
\;>\;0
\quad\text{for }j= s+1,\dots,t,
\]
and $\Delta_j\geq 0$ otherwise (in fact $\Delta_j=0$ for $\tilde\alpha=1$ when $j\not\in\{s+1,s+2,\ldots, t\}$). 

Then note that
\begin{align*}
    \mathbb{E}[|\Pi(u)-s|^{\tilde\alpha}]-\mathbb{E}[|\Pi(u)-t|^{\tilde\alpha}]
&=\sum_{j=1}^nD(j)\,\P_{\alpha_0,\beta_0,\sigma_0}(\Pi(u)=j)
\\&=\sum_{j=2}^n\bigl[D(j)-D(j+1)\bigr]\,\P_{\alpha_0,\beta_0,\sigma_0}(\Pi(u)<j)
\\&=-\sum_{j=s+1}^t\Delta_{j+1}\P_{\alpha_0,\beta_0,\sigma_0}(\Pi(u)<j).
\end{align*}
Applying this for $u=i$ and $u=k$ gives
\[
F(x)-F(x')
=\sum_{j=s+1}^t\Delta_{j+1}\Bigl[\P_{\alpha_0,\beta_0,\sigma_0}(\Pi(k)<j)-\P_{\alpha_0,\beta_0,\sigma_0}(\Pi(i)<j)\Bigr].
\]

Since $\Delta_j\geq 0$ and it strictly positive for $j\in{s+1,...,t}$, it remains to show $\P_{\alpha_0,\beta_0,\sigma_0}(\Pi(k)<j)\ge \P_{\alpha_0,\beta_0,\sigma_0}(\Pi(i)<j)$ for each $j$.  To see this, define an involution $\phi$ on $S_n$ by swapping the values in positions $k$ and $i$, where $k<i$.  Whenever $\pi$ satisfies
$\pi(k)< \pi(i)$,
by repeating the argument above,
\[
d_{\alpha_0}\bigl(\phi(\pi),\mathrm{id}\bigr)
-d_{\alpha_0}(\pi,\mathrm{id})
=\bigl|\pi(i)-k\bigr|^{\alpha_0}-\bigl|\pi(i)-i\bigr|^{\alpha_0}
\;-\;(\bigl|\pi(k)-k\bigr|^{\alpha_0}-\bigl|\pi(k)-i\bigr|^{\alpha_0})
\;>\;0,
\]
so $\P_{\alpha_0,\beta_0,\sigma_0}(\Pi=\phi(\pi))<\P_{\alpha_0,\beta_0,\sigma_0}(\Pi=\pi)$.  Hence summing over all $\pi$ with $\pi(k)<j\le\pi(i)$ yields
\[
\P_{\alpha_0,\beta_0,\sigma_0}\bigl(\Pi(k)<j,\;\Pi(i)\ge j\bigr)
> \P_{\alpha_0,\beta_0,\sigma_0}\bigl(\Pi(i)<j,\;\Pi(k)\ge j\bigr).
\]
Adding $\P_{\alpha_0,\beta_0,\sigma_0}\Big(\Pi(k)<j,\Pi(i)<j\Big)$ to both sides gives 
$\P_{\alpha_0,\beta_0,\sigma_0}(\Pi(k)<j)>\P_{\alpha_0,\beta_0,\sigma_0}\Big(\Pi(i)<j\Big)$.  
Therefore, each term in the sum is nonnegative, and at least one is strictly positive, so $F(x)-F(x')>0$.  This completes the proof of the first part.

\noindent\textbf{Step 2: Eliminating inversions via successive swaps.}  
We now transform any \(\mathbf x\neq \sigma_0\) to $\sigma_0$ by a finite sequence of single-inversion swaps, each of which strictly decreases \(F(x)\).  To see this, not that
if \(\mathbf x\neq\sigma_0\) (recall that we assume $\sigma_0$ is the identity), then the sequence \((x_1,\dots,x_n)\) is not increasing, so there exists at least one pair of indices
\[
1\le k<i\le n
\quad\text{with}\quad
x_k > x_i.
\]
By Step 1, swapping the two entries at positions \(k\) and \(i\) produces a new permutation
\(
\mathbf x'\)
with
$F(\mathbf x')<F(\mathbf x),$
because \((k,i)\) was an inversion.
Now,  by repeating this argument, since there are only finitely many permutations and \(F\) strictly decreases at each step, the process terminates in a finite number of swaps, necessarily at the unique no-inversion state \(\mathbf x=\sigma_0\).

Putting this all together, we conclude that any \(\mathbf x\neq\sigma_0\) can be changed to get a lower $F$ value, so
\[
\sigma_0=\arg\min_{\mathbf x\in S_n}F(\mathbf x)
=\arg\min_{\sigma'\in S_n}\mathbb{E}_{{\alpha_0,\beta_0,\sigma_0}}\bigl[d_{\tilde \alpha}(\Pi,\sigma')\bigr],
\]
and the minimizer is unique.  
\end{proof}

Now, we are ready to finish the proof of Lemma~\ref{lm:unqiue minimizer} (and hence the first part of Theorem~\ref{thm:main_mle}).
 Define the empirical objective
\[
F_m(\sigma')
\;=\;
\frac{1}{m}\sum_{j=1}^m d_1(\pi^{(j)},\sigma'),
\]
and let
\[
\hat\sigma_m
\;=\;
\arg\min_{\sigma'\in S_n}F_m(\sigma')
\quad(\text{ties broken arbitrarily}).
\]
We want to prove that $\hat\sigma_m\to\sigma$ almost surely.

Since $S_n$ is finite, the Strong Law of Large Numbers gives, for each $\sigma'\in S_n$,
\[
F_m(\sigma')
=\frac1m\sum_{j=1}^m d_{\tilde\alpha}(\pi^{(j)},\sigma')
\;\xrightarrow[]{\mathrm{a.s.}}\;
F(\sigma')
=\mathbb{E}_{\Pi\sim\P_{\alpha_0,\beta_0,\sigma_0}}\bigl[d_{\tilde \alpha}(\Pi,\sigma')\bigr].
\]
Since there are only $n!$ different permutations, the convergence is uniform:
\[
\sup_{\sigma'\in S_n}\bigl|F_m(\sigma')-F(\sigma')\bigr|
\;\xrightarrow[]{\mathrm{a.s.}}\;0.
\]
Since $F$ has a unique minimizer at $\sigma_0$ (as proved above), set
\(
\delta=\min_{\sigma'\neq\sigma_0}[\,F(\sigma')-F(\sigma_0)\,]>0.
\)
By uniform convergence, for large enough $m$ we have
\(\sup_{\sigma'}|F_m(\sigma')-F(\sigma')|<\delta/2\). Therefore, for all $\sigma'\neq\sigma_0$,
\[
F_m(\sigma_0)
\;\le\;
F(\sigma_0)+\tfrac\delta2< F(\sigma')-\tfrac\delta2
\;\le\;
F_m(\sigma').
\]
Thus for large enough $m$ the empirical minimizer must converge to the true ranking
($\hat\sigma_m\to\sigma_0$).  This proves almost sure convergence $\hat\sigma_m\overset{a.s.} \to \sigma_0$. 

For the tail bound, we have a crude bound of $0\le d_{\tilde\alpha}(\pi,\sigma')\le n^{\tilde \alpha+1}$, so by Hoeffding’s inequality, for any fixed $\sigma'\neq\sigma_0$,
\[
\P\Bigl(\bigl|F_m(\sigma')-F(\sigma')\bigr|\ge \tfrac\delta2\Bigr)
\;\le\;
2\exp\!\Bigl(-\tfrac{2m(\delta/2)^2}{(n^{\tilde\alpha+1})^2}\Bigr)
\;=\;
2\exp\!\Bigl(-\tfrac{m\,\delta^2}{2\,n^{2\tilde\alpha+2}}\Bigr),
\]
here the probability is over $m$ i.i.d samples from $\P_{\alpha_0,\beta_0,\sigma_0}$
Applying the union bound over all $|S_n|=n!$ choices of $\sigma'$,
\[
\P\bigl(\sup_{\sigma'}|F_m(\sigma')-F(\sigma')|>\delta/2\bigr)
\;\le\;
n!\;\cdot\;\exp\!\Bigl(-\tfrac{m\,\delta^2}{2\,n^{2\tilde\alpha+2}}\Bigr)
\]
As a result,
\[
\P\bigl(\hat\sigma_m\neq\sigma_0\bigr)
\;\le\;
n!\;\cdot\;\exp\!\Bigl(-\tfrac{m\,\delta^2}{2\,n^{2\tilde\alpha+2}}\Bigr)
\]
which in particular decays to $0$ exponentially in $m$. 

To finish the proof of the first part, it remains to find lower bound for
\[
\delta = \min_{\sigma' \neq \sigma_0} \big[ F(\sigma') - F(\sigma_0) \big] > 0.
\]

Without loss of generality, assume that $\sigma_0$ is the identity permutation. As shown earlier, each transposition strictly decreases $F$, so it suffices to consider only transpositions. Let $\sigma' = (i\;j)$ be a transposition, and define $h := |i - j| \ge 1$ as the separation between the swapped positions. The expected cost difference decomposes as:
\[
\delta_{i,j} := \mathbb{E}\left[ |\Pi(i) - j|^{\tilde\alpha} - |\Pi(i) - i|^{\tilde\alpha} \right]
+ \mathbb{E}\left[ |\Pi(j) - i|^{\tilde\alpha} - |\Pi(j) - j|^{\tilde\alpha} \right].
\]
Then $\delta \ge \min_{i < j} \delta_{i,j}$.

Since $\Pi(i), \Pi(j), i, j \in \{1,\dots,n\}$ are integers, the function $|\cdot|^{\tilde\alpha}$ is evaluated only on integer arguments. The minimum nonzero difference of the form $|x - j|^{\tilde\alpha} - |x - i|^{\tilde\alpha}$ must be at least
\[
\min_{ij}\delta_{i,j}\geq\min_{0 \le z \le n-1} \left[(z+1)^{\tilde\alpha} - z^{\tilde\alpha} \right] = n^{\tilde\alpha} - (n-1)^{\tilde\alpha},
\]
since the function $z \mapsto (z+1)^{\tilde\alpha} - z^{\tilde\alpha}$ is decreasing in $z$ for $\tilde\alpha > 0$.
By using the fact that $n^{\tilde\alpha} - (n-1)^{\tilde\alpha}\geq \tilde\alpha(n-1)^{\tilde\alpha-1}$ and that $(1-\frac{1}{n})^{2\alpha}\geq 1-\frac{2\alpha}{n}-\frac{2\alpha^2}{n^2}$, we the desired inequality:
\[
\P\bigl(\hat\sigma_m\neq\sigma_0\bigr)
\;\le\;
n!\;\cdot\;\exp\!\Bigl(-m\big(\tfrac{\tilde\alpha^2}{2\,n^{4}}(1-\frac{2\tilde\alpha}{n})\big)\Bigr).
\]
If we let $\tilde\alpha=1$, we get the desired inequality.

\subsection{Consistency of the MLE for \(\hat\alpha_m,\hat\beta_m\) (Part 2 of Theorem~\ref{thm:main_mle})}

Throughout this section, we assume $\hat\sigma_m$ is estimated in the first step of MLE. Define the normalized likelihood function:
\[
\ell_m(\alpha,\beta)
\;=\;
\frac{1}{m}\sum_{j=1}^m\log f\bigl(\pi^{(j)};\alpha,\beta,\hat\sigma_m\bigr)
\;=\;
-\,\beta\;\frac1m\sum_{j=1}^m d_\alpha(\pi^{(j)},\hat\sigma_m)
\;-\;\ln Z(\alpha,\beta),
\]
and let 
\(\hat\theta_m=(\hat\alpha_m,\hat\beta_m)\)
be any maximizer of \(\ell_m\) over 
\(\Theta=(0,\infty)\times(0,\infty)\).   Further  define,
\[
\Psi(\theta):=\E[\Psi_m(\theta)],
\quad
\text{where }\theta=(\alpha,\beta).
\]
Recall the definition of $\Psi_m$ from Step 2 of Algorithm~\ref{alg:onepass_mle}.
We verify two key properties:
\begin{enumerate}
  \item  {\it Uniform convergence:}
    \[
      \sup_{\theta\in\Theta}
      \bigl\|\Psi_m(\theta)-\Psi(\theta)\bigr\|
      \;\xrightarrow{\P}\;0.
    \]
  \item  {\it Identifiability of the zero:} Let $\theta_0=(\alpha_0,\beta_0)$
    \(\Psi(\theta_0)=0\), and for every \(\varepsilon>0\),
    \[
       \inf_{\|\theta-\theta_0\|\ge\varepsilon}
       \bigl\|\Psi(\theta)\bigr\|
      >0.
    \]
\end{enumerate}
Once these two properties are proved, Theorem 5.9 of \cite{van2000asymptotic} yields  
\(\hat\theta_m\xrightarrow{\P}\theta_0=(\alpha_0,\beta_0)\). We start to prove identifiability.

\begin{lemma}(Identifiability)\label{lm: identifiable}
The population score \(\Psi(\theta)\) satisfies
\(\Psi(\theta_0)=0\) and for every \(\varepsilon>0\),
\[
\inf_{\|\theta-\theta_0\|\ge\varepsilon}
\|\Psi(\theta)\|
>0.
\]
\end{lemma}

\begin{proof}
A straightforward calculation gives
\[
\Psi(\theta)
=
\begin{pmatrix}
\mathbb{E}_{\alpha_0,\beta_0}\left[\dot{d}_\alpha(\pi)\right] - \mathbb{E}_{\alpha,\beta}\left[\dot{d}_\alpha(\pi)\right]
\\[6pt]
\mathbb{E}_{\alpha_0,\beta_0}\left[d_\alpha(\pi)\right] - \mathbb{E}_{\alpha,\beta}\left[d_\alpha(\pi)\right]
\end{pmatrix},
\]
here since $\hat\sigma_m$ is fixed, with abuse of notation we write $d_\alpha(\pi)$ instead of $d_\alpha(\pi,\hat\sigma_m)$ and $\E_{\alpha',\beta'}$ instead of $\E_{\P{_\alpha',\beta',\hat \sigma_m}}$.
With this, proving identifiability reduces to showing that the system of equations
\[
\mathbb{E}_{\alpha,\beta}[d_\alpha(\pi)] = \mathbb{E}_{\alpha_0,\beta_0}[d_\alpha(\pi)],
\quad
\mathbb{E}_{\alpha,\beta}[\dot{d}_\alpha(\pi)] = \mathbb{E}_{\alpha_0,\beta_0}[\dot{d}_\alpha(\pi)]
\]
admits a unique solution \((\alpha,\beta) = (\alpha_0,\beta_0)\).

To facilitate the argument, we fix \((\alpha,\beta)\) and view the left-hand sides as functions of \((\alpha_0,\beta_0)\). In this formulation, identifiability reduces to proving that for each fixed \((\alpha,\beta)\), there exists a unique \((\alpha_0,\beta_0)\) satisfying the system.

For this purpose, we introduce the auxiliary map
\[
\boldsymbol{f}_{\alpha}(\alpha_0,\beta_0)
=
\left(
\mathbb{E}_{{\alpha_0,\beta_0}}[d_{\alpha}(\pi)],\;
\mathbb{E}_{{\alpha_0,\beta_0}}[\dot{d}_{\alpha}(\pi,)]
\right),
\]
where \(\alpha \) is fixed. That is, while the distance function $d_\alpha$ is fixed, the true data-generating parameters \((\alpha_0,\beta_0)\) varies.
The key idea is that if \(\boldsymbol{f}_{\alpha}\) is injective, then the solution to the system is unique. We now proceed to prove that \(\boldsymbol{f}_{\alpha}\) is a global diffeomorphism onto its image. To this end, we first establish local invertibility and properness, which together imply global invertibility by the Hadamard--Caccioppoli theorem.

\begin{lemma}[Local Invertibility]\label{lem:local_invertibility}
For any \(\alpha > 0\), the function \(\boldsymbol{f}_{\alpha}\) is locally invertible on \((0,\infty)^2\).
\end{lemma}

\begin{proof}
It suffices to show that the Jacobian matrix \(\boldsymbol{J}\boldsymbol{f}_{\alpha}(\alpha_0,\beta_0)\) is invertible at every point \((\alpha_0,\beta_0) \in (0,\infty)^2\).
Direct calculation gives:
\[
\boldsymbol{J}\boldsymbol{f}_{\alpha}(\alpha_0,\beta_0)
= -
\begin{bmatrix}
\mathrm{Cov}_{\alpha_0,\beta_0}\left(d_\alpha(\pi), d_{\alpha_0}(\pi)\right) &
\beta \, \mathrm{Cov}_{\alpha_0,\beta_0}\left(d_\alpha(\pi), \dot{d}_{\alpha_0}(\pi)\right)
\\[6pt]
\mathrm{Cov}_{\alpha_0,\beta_0}\left(\dot{d}_\alpha(\pi), d_{\alpha_0}(\pi)\right) &
\beta \, \mathrm{Cov}_{\alpha_0,\beta_0}\left(\dot{d}_\alpha(\pi), \dot{d}_{\alpha_0}(\pi)\right)
\end{bmatrix}.
\]

Define the two-dimensional random vectors:
\[
\boldsymbol{a}_1(\pi) = \left(d_{\alpha_0}(\pi), \dot{d}_{\alpha_0}(\pi)\right),
\quad
\boldsymbol{a}_2(\pi) = \left(d_{\alpha}(\pi), \dot{d}_{\alpha}(\pi)\right).
\]
Then
\[
\det\left( \boldsymbol{J}\boldsymbol{f}_{\tilde\alpha}(\alpha,\beta) \right)
= \beta \cdot \det\left( \mathrm{Cov}_{\alpha,\beta}\left(\boldsymbol{a}_1(\pi), \boldsymbol{a}_2(\pi)\right) \right).
\]

When \((\alpha_0,\beta_0) = (\alpha,\beta)\), we have \(\boldsymbol{a}_1(\pi) = \boldsymbol{a}_2(\pi)\), and the covariance matrix reduces to the variance of \(\boldsymbol{a}_1(\pi)\), which is positive definite.

When \(\alpha_0 \neq \alpha\),  
suppose, for contradiction, that \(\mathrm{Cov}_{\alpha_0,\beta_0}(\boldsymbol{a}_1(\Pi), \boldsymbol{a}_2(\Pi))\) is singular. Then there must exist nonzero vectors \(\lambda^{(1)}, \lambda^{(2)} \in \mathbb{R}^2\) such that
\begin{equation}\label{eq:invertibility_f_jacobian}
\left\langle \lambda^{(1)}, \boldsymbol{a}_1(\Pi) \right\rangle
+
\left\langle \lambda^{(2)}, \boldsymbol{a}_2(\Pi) \right\rangle
\quad \text{is constant almost surely}.
\end{equation}
Since \(\P_{\alpha_0,\beta_0}\) has full support on the finite set \(S_n\), this equality holds for every \(\pi \in S_n\).

To derive a contradiction, we consider a few examples. Let \(\pi_1=(2,4)\) denote the transposition swapping elements \(2\) and \(4\), and \(\pi_2\) denote the composition of \(\pi_1\) with the transposition swapping \(1\) and \(3\), that is, \(\pi_2 = (1\;3)(2\;4)\).
Applying \eqref{eq:invertibility_f_jacobian} to both \(\pi_1\) and \(\pi_2\), we have:
\begin{align*}
\left\langle \lambda^{(1)}, \boldsymbol{a}_1(\pi_1) \right\rangle + \left\langle \lambda^{(2)}, \boldsymbol{a}_2(\pi_1) \right\rangle &= 
\left\langle \lambda^{(1)}, \boldsymbol{a}_1(\pi_2) \right\rangle + \left\langle \lambda^{(2)}, \boldsymbol{a}_2(\pi_2) \right\rangle.
\end{align*}
By construction, we observe:
\[
\boldsymbol{a}_1(\pi_2) = \boldsymbol{a}_1(\pi_1) + \boldsymbol{a}_1((1\;3)),
\quad
\boldsymbol{a}_2(\pi_2) = \boldsymbol{a}_2(\pi_1) + \boldsymbol{a}_2((1\;3)).
\]
Substituting the expressions for \(\pi_2\), we obtain:
\[
\left\langle \lambda^{(1)}, \boldsymbol{a}_1((1\;3)) \right\rangle
+
\left\langle \lambda^{(2)}, \boldsymbol{a}_2((1\;3)) \right\rangle
= 0.
\]

Repeating this argument with other simple transpositions such as \((1\;4)\), \((1\;5)\), and \((1\;6)\), we obtain a homogeneous linear system:
\[
\boldsymbol{A}
\begin{pmatrix}
\lambda^{(1)} \\ \lambda^{(2)}
\end{pmatrix}
= 0,
\]
where
\[
\boldsymbol{A}
=
\begin{bmatrix}
\boldsymbol{a}_1(1\;3) & \boldsymbol{a}_2(1\;3) \\
\boldsymbol{a}_1(1\;4) & \boldsymbol{a}_2(1\;4) \\
\boldsymbol{a}_1(1\;5) & \boldsymbol{a}_2(1\;5) \\
\boldsymbol{a}_1(1\;6) & \boldsymbol{a}_2(1\;6)
\end{bmatrix}.
\]
It can be checked that when \(\alpha \neq \alpha_0\), the matrix \(\boldsymbol{A}\) has full rank, 
Since the points \((d_{\alpha_0}(\pi), \dot{d}_{\alpha_0}(\pi), d_{\alpha}(\pi), \dot{d}_{\alpha}(\pi))\) for different transpositions \((1\;k)\) involve distinct powers of \(|k-1|\) under \(\alpha\) and \(\alpha_0\), and the functions \(\alpha \mapsto |k-1|^\alpha\) and \(\alpha \mapsto \log|k-1|\cdot |k-1|^\alpha\) are linearly independent for different \(k\), the matrix \(\boldsymbol{A}\) is full rank whenever \(\alpha \neq \alpha_0\).
As a result, \(\lambda^{(1)} = \lambda^{(2)} = 0\).
Thus, no nontrivial linear relation exists, contradicting the assumption that \(\mathrm{Cov}_{\alpha_0,\beta_0}(\boldsymbol{a}_1(\Pi), \boldsymbol{a}_2(\Pi))\) is singular. This completes the proof of local invertibility.
\end{proof}

\bigskip

\begin{lemma}[Properness]\label{lem:properness}
For any \(\alpha\), the function \(\boldsymbol{f}_{\alpha}\) is proper: the preimage of every compact set is compact.
\end{lemma}

\begin{proof}
Let \(K\subset \mathrm{Img}(\boldsymbol{f}_{\tilde\alpha})\) be any compact set.  
We must show that \(\boldsymbol{f}_{\tilde\alpha}^{-1}(K)\) is compact.
Suppose for contradiction that \(\boldsymbol{f}_{\tilde\alpha}^{-1}(K)\) is not compact.  
Then there exists a sequence \((\alpha_m,\beta_m)\) in \(\boldsymbol{f}_{\tilde\alpha}^{-1}(K)\) that escapes to infinity, meaning
\[
\lim_{m\to\infty} \|(\alpha_m,\beta_m)\| = \infty.
\]

As \((\alpha_m,\beta_m)\to\infty\), the corresponding distributions \(\P_{\alpha_m,\beta_m}\) concentrate on the identity permutation. To see this as \((\alpha_m,\beta_m) \to \infty\), either \(\alpha_m \to \infty\) or \(\beta_m \to \infty\) or both. In any of these cases, for any permutation \(\pi \neq \mathrm{id}\), we have
\[
\frac{\P_{\alpha_m,\beta_m}(\pi)}{\P_{\alpha_m,\beta_m}(\mathrm{id})} = \exp\left( -\beta_m ( d_{\alpha_m}(\pi) - d_{\alpha_m}(\mathrm{id}) ) \right) \to 0.
\] Thus, as $m$ increases,
\[
\boldsymbol{f}_{\tilde\alpha}(\alpha_m, \beta_m) \to \boldsymbol{0}.
\]
But note that for any finite \((\alpha,\beta) \in (0,\infty)^2\), we always have:
\[
\mathbb{E}_{\alpha,\beta}[d_{\tilde\alpha}(\Pi)] > 0,
\quad
\mathbb{E}_{\alpha,\beta}[\dot{d}_{\tilde\alpha}(\Pi)] > 0,
\]
since \(\P_{\alpha,\beta}\) has full support on \(S_n\), and \(d_{\tilde\alpha}(\Pi) > 0\) with positive probability.
Therefore, \(\boldsymbol{0} \notin \mathrm{Img}(\boldsymbol{f}_{\tilde\alpha})\).
Now, since \(\boldsymbol{f}_{\tilde\alpha}(\alpha_m, \beta_m) \in K\) for all \(m\), and \(\boldsymbol{f}_{\tilde\alpha}(\alpha_m, \beta_m) \to \boldsymbol{0}\), we conclude that \(\boldsymbol{0} \in K\) by closedness of \(K\).
This contradicts \(\boldsymbol{0} \notin \mathrm{Img}(\boldsymbol{f}_{\tilde\alpha})\). Therefore, the preimage \(\boldsymbol{f}_{\tilde\alpha}^{-1}(K)\) must be compact.

\end{proof}

Now, to finish the proof of uniqueness of the critical point, observe that the condition \(\Psi(\theta) = 0\) is equivalent to
\[
\boldsymbol{f}_\alpha(\alpha,\beta) = \boldsymbol{f}_\alpha(\alpha_0,\beta_0).
\]
Since \(\boldsymbol{f}_\alpha\) is a global diffeomorphism (by Lemmas~\ref{lem:local_invertibility} and~\ref{lem:properness}), it follows that \((\alpha,\beta) = (\alpha_0,\beta_0)\), and hence \(\Psi(\theta_0) = 0\) has a unique solution.

To complete the identifiability condition, we note that the global invertibility of \(\boldsymbol{f}_\alpha\) implies that its inverse \(\boldsymbol{f}_\alpha^{-1}\) is continuous on its image. Therefore, \(\Psi(\theta) = \boldsymbol{f}_\alpha(\alpha_0, \beta_0) - \boldsymbol{f}_\alpha(\theta)\) is continuous, vanishes only at \(\theta_0\), and cannot tend to zero along any sequence staying at distance at least \(\varepsilon > 0\) from \(\theta_0\). Thus, for every \(\varepsilon > 0\),
\[
\inf_{\|\theta - \theta_0\| \ge \varepsilon} \|\Psi(\theta)\| > 0,
\]
as required.
\end{proof}

\medskip
\paragraph{Uniform Convergence of $\Psi_m$.}
Next we prove uniform convergence of $\Psi_m$ to $\Psi$. Note that if the true-parameters of Mallows $(\alpha_0,\beta_0)$ where restricted to a compact set, then this was immediate. But the proof for general parameters needs careful analysis which we bring next.

First, fix any compact rectangle
\[
\Theta_R := [1/R, R] \times [1/R, R] \subset (0,\infty)\times(0,\infty),
\]
for some \(R>1\).
Over \(\Theta_R\), \(\Psi_m(\theta)\) and \(\Psi(\theta)\) are uniformly bounded and equicontinuous as functions of \(\theta\). 
Thus, by pointwise convergence along with the law of large numbers, we have
\[
\sup_{\theta\in\Theta_R} \|\Psi_m(\theta) - \Psi(\theta)\| \xrightarrow{\P} 0.
\]

Now, outside \(\Theta_R\) (that is, for \(\theta\notin\Theta_R\)), we claim that \(\|\Psi(\theta)\|\) is bounded away from zero for large \(R\), and hence the sequence $\theta_m$ lies inside $\Theta_R$ for some large $R$.
Indeed, as \(\|\theta\|\to\infty\), we know from the properness of \(\boldsymbol{f}_\alpha\)  (Lemma~\ref{lem:properness}) that \(\|\Psi(\theta)\|\to\infty\) which implies that
\[
\inf_{\theta\notin\Theta_R} \|\Psi(\theta)\| \geq \delta_R > 0
\]
for some \(\delta_R > 3\varepsilon\) (depending on \(R\)).
Now, consider any \(\theta\notin\Theta_R\).
By the pointwise convergence of \(\Psi_m(\theta) \to \Psi(\theta)\) as \(m\to\infty\) for each fixed \(\theta\).
Thus, for each fixed \(\theta\notin\Theta_R\),  
given any small \(\delta > 0\), there exists \(m_0(\theta,\delta)\) such that for all \(m \geq m_0\),
\[
\mathbb{P}\left( \|\Psi_m(\theta) - \Psi(\theta)\| \leq \varepsilon \right) \geq 1-\delta.
\]
Moreover, since \(\|\Psi(\theta)\| \geq 3\varepsilon\),  
if \(\|\Psi_m(\theta) - \Psi(\theta)\| \leq \varepsilon\), then by the triangle inequality:
\[
\|\Psi_m(\theta)\| \geq \|\Psi(\theta)\| - \|\Psi_m(\theta) - \Psi(\theta)\| \geq 2\varepsilon.
\]
Therefore $\Psi_m(\theta)\neq 0$ for $\theta\not\in \Theta_R$ and all $m$ large enough. But by the first order optimality conditions we must have $\Psi_m(\theta_m)=0$. As a result we must have the sequence is contained in a compact set  i.e., for some large enough $R$, we have that $\theta_m\in\Theta_R$ with high probability as $m\to \infty$. Since $\Psi_m$ has no minimizer outside $\Theta_R$ then by applying Theorem 5.9 of van der Vaart~\cite{van2000asymptotic}, we conclude that
\[
\hat\theta_m = (\hat\alpha_m, \hat\beta_m) \xrightarrow{\P} (\alpha_0, \beta_0).
\]

\paragraph{Asymptotic Normality.} The final step to prove asymptotic normality by  Theorem 5.23 of van der Vaart~\cite{van2000asymptotic}. We verify the conditions of this theorem for the M-estimator \((\hat\alpha_m, \hat\beta_m)\), defined as the maximizer of the empirical objective \(\ell_m(\theta)\). 

\textit{(i) Consistency:} Established previously using Theorem 5.9.

\textit{(ii) Differentiability:} The log-likelihood \(\ell_m(\theta)\) is twice continuously differentiable in \(\theta = (\alpha, \beta)\) because both \(d_\alpha(\pi, \sigma_0)\) and \(Z(\alpha, \beta)\) are smooth in \(\alpha\) and \(\beta\).

\textit{(iii) Central Limit Theorem for the Score:} 
By the classical central limit theorem, the empirical average of the first term of $\Psi_m$ converges in distribution:
\[
\sqrt{m} \left( \frac{1}{m} \sum_{j=1}^m \begin{pmatrix} \beta_0 \, \dot{d}_{\alpha_0}(\pi^{(j)}) \\ d_{\alpha_0}(\pi^{(j)}) \end{pmatrix} - \mathbb{E} \left[ \begin{pmatrix} \beta_0 \, \dot{d}_{\alpha_0}(\Pi) \\ d_{\alpha_0}(\Pi) \end{pmatrix} \right] \right)
\xrightarrow{d} \mathcal{N}(0, B),
\]
where \(B\) is the variance of the score under \(\mathbb{P}_{\alpha_0, \beta_0, \sigma_0}\).

\textit{(iv) Hessian convergence:} The Hessian \(\nabla^2_\theta \ell_m(\theta)\) is the sum of an empirical average and \(\nabla^2 \ln Z(\theta)\), both of which converge uniformly by the law of large numbers and analytic smoothness of \(Z\).

\textit{(v) Invertibility:} The matrix \(\mathcal{I}(\alpha_0, \beta_0) = -\nabla^2 \ell(\theta_0)\) is positive definite due to identifiability of $\Psi$ (Lemma~\ref{lm: identifiable}).
Thus, all conditions of Theorem 5.23 are satisfied, and the result follows. \qed

\subsection{Asymptotic Normality under Approximate Score}

\begin{proof}[Proof of Theorem~\ref{thm:approximate mle}]
Let \(\hat\theta_m\) be the exact MLE, satisfying \(\Psi_m(\hat\theta_m) = 0\). Using Theorem~\ref{thm:main_mle},
\[
\sqrt{m}(\hat\theta_m - \theta_0) \xrightarrow{d} \mathcal{N}(0, \mathcal{I}^{-1}).
\]

We now analyze the perturbation due to the approximate score:
\[
\widetilde{\Psi}_m(\tilde\theta_m) = \Psi_m(\tilde\theta_m) + \Delta(\tilde\theta_m) = 0.
\]Since \(\Psi_m(\hat\theta_m) = 0\) and \(\widetilde{\Psi}_m(\tilde\theta_m) = 0\), we write:
\[
\widetilde{\Psi}_m(\tilde\theta_m) = \Psi_m(\tilde\theta_m) + \Delta(\tilde\theta_m) = 0,
\]
which implies:
\[
\Psi_m(\tilde\theta_m) = -\Delta(\tilde\theta_m).
\]

Now apply a first-order Taylor expansion of \(\Psi_m\) around \(\hat\theta_m\):
\[
\Psi_m(\tilde\theta_m) = \Psi_m(\hat\theta_m) + \nabla \Psi_m(\hat\theta_m)(\tilde\theta_m - \hat\theta_m) + R_m,
\]
where \(R_m\) is the second-order remainder:
\[
R_m = o(\|\tilde\theta_m - \hat\theta_m\|).
\]
Since \(\Psi_m(\hat\theta_m) = 0\), this gives:
\[
\nabla \Psi_m(\hat\theta_m)(\tilde\theta_m - \hat\theta_m) = -\Delta(\tilde\theta_m) + o(\|\tilde\theta_m - \hat\theta_m\|).
\]

By assumption, \(\nabla \Psi_m(\hat\theta_m) \xrightarrow{P} \mathcal{I}(\theta_0)\), which is invertible. Thus, for large enough \(m\), we can solve:
\[
\tilde\theta_m - \hat\theta_m = -\nabla \Psi_m(\hat\theta_m)^{-1}\Delta(\tilde\theta_m) + o(\|\tilde\theta_m - \hat\theta_m\|).
\]
Using the bound \(\|\Delta(\tilde\theta_m)\| = o(m^{-1/2})\), we conclude:
\[
\|\tilde\theta_m - \hat\theta_m\| = o_P(m^{-1/2}) \Rightarrow \sqrt{m}(\tilde\theta_m - \hat\theta_m) = o_P(1).
\]

Combining with the asymptotic normality of \(\hat\theta_m\),
\[
\sqrt{m}(\tilde\theta_m - \theta_0) = \sqrt{m}(\hat\theta_m - \theta_0) + o_P(1) \xrightarrow{d} \mathcal{N}(0, \mathcal{I}^{-1}).
\]
\end{proof}





   


\section{Proofs for Efficient Sampling}

\subsection{Geometric Decay of Marginal Distributions}\label{proof: marginal}
The main goal of this section is proving Lemma~\ref{lm:Pij_decay}. We break the argument into a few steps.
The first observation is that the marginal probability is symmetric in $i$ and $j$, i.e., we claim that
\[
P_n(i,j) \;=\; P_n\bigl(n+1-j, \,n+1-i\bigr).
\] This is true because of a standard relabeling argument implies that the probability of having $\pi(i) = j$ is the same as the probability of having $\pi(n+1-j) = n+1-i$ once we account for symmetry in $d_\alpha$. Therefore, without loss of generality, it is enough to prove that  for $j-i>k,$
\[\frac{P_n(i, j+1)}{P_n(i, j)} \leq c(\alpha, \beta).\]

\medskip

To proceed with the rest of the proof, we need the following definitions. 
Define the sets of permutations:
\[
S^{(i,j)} \;:=\; \{\pi\in S_n : \pi(i)=j\} 
\quad \text{and} \quad
S^{(i,j+1)} \;:=\; \{\pi \in S_n: \pi(i)=j+1\}.
\]
and the weight function \( w: S_n \to \mathbb{R}^+ \) for a permutation \( \pi \), which assigns to each permutation \( \pi \) its corresponding probability weight:
\begin{equation}
    w(\pi) := e^{-\beta d_\alpha(\pi, \text{id})} = \prod_{i=1}^n e^{-\beta |i - \pi(i)|^\alpha}.
\end{equation}
For any subset \( S \subseteq S_n \), we extend this definition, with a slight abuse of notation, to the total weight of \( S \) as
$w(S) = \sum_{\pi \in S} w(\pi).$

In broad terms, the proof compares the weights of two sets of permutations $S^{(i,j)}$ and $S^{(i,j+1)}$.
In fact, \(\frac{P_n(i,j+1)}{P_n(i,j)}=\frac{w(S^{(i,j+1)})}{w(S^{(i,j)})}\). Thus, our objective is to prove
\(\displaystyle 
\frac{w(S^{(i,j+1)})}{w(S^{(i,j)})} < c(\alpha,\beta)
\)
for large \( j - i \). For this purpose we find mappings from permutations in $w(S^{(i,j+1)})$ to $w(S^{(i,j)})$ that shows the total weight ratio is bounded.

\begin{itemize}
\item \textbf{First Round of Mappings (Proposition~\ref{prop:first-mapping}).}  
We break down the set \(S^{(i,j+1)}\) by looking at which index was mapped to \(j\). We build a mapping (a “swap”) that sends permutations in \(S^{(i,j+1)}_{(l,j)}:= \{\pi\in S_n: \pi(i)=j+1, \pi(l)=j\}\) to the permutations in \(S^{(i,j)}_{(l,j+1)}:=\{\pi\in S_n: \pi(i)=j, \pi(l)=j+1\}\).  \\
  - If \(l > j\), we show the weight increases -- directly helping us prove \(S^{(i,j+1)}\) is smaller in total weight.\\  
  - If \(l \le j\), the mapping might decrease the weight, so we cannot immediately conclude a reduction. 

\item \textbf{Second Round of Mappings (Proposition~\ref{prop:second-mapping}).}  
For the ``problematic'' subset where \(l \le j\), we employ a second mapping by identifying the first index (larger than $j$) mapped to a value {below} $j$.  Swapping this index with $l$ brings the resulting permutation closer to \emph{identity}, leading to an increased weight. 
\end{itemize}
Finally, by uniting these two mapping arguments, we conclude \(\tfrac{P_n(i,j+1)}{P_n(i,j)} < 1\). 
We now present the two mapping steps in detail.

\subsubsection{First Round of Mappings}

The first round of mapping is described above is done by swapping inverse of $j$ and $j+1$ which gives the following bounds.
\begin{proposition}
\label{prop:first-mapping}
Suppose $j > i$, and let $S^{(i,j+1)}_{(l,j)}$ and $S^{(i,j)}_{(l,j+1)}$ be defined as above.  Then for any $l\in[n]\setminus\{i\}$, and any $k\leq j-i$,
\[\frac{w(S^{(i,j+1)}_{(l,j)})}{w(S^{(i,j)}_{(l,j+1)})}\leq \begin{cases}
    e^{-\beta((k+1)^\alpha-k^\alpha)}e^{-\beta((j+1-l)^\alpha-(j-l)^\alpha)} & l>j\\
e^{-\beta((k+1)^\alpha-k^\alpha)}e^{\beta((j+1-l)^\alpha-(j-l)^\alpha)} &  l\leq j
\end{cases}.\]
\end{proposition}

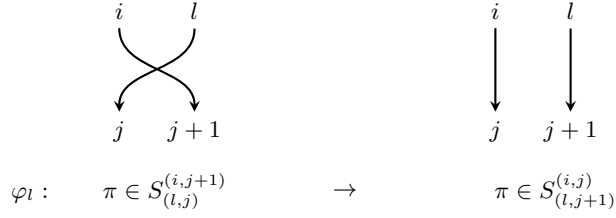
\begin{figure}
\centering\begin{tikzpicture}[%
  x=2.0cm,      
  y=1.6cm,      
  >=stealth,    
  font=\small
]

%
\node (i)    at (-.5,0) {$i$};
\node (l)    at (0,0) {$l$};
\node (i2) at (2,0) {$i$};
\node (l2) at (2.5,0) {$l$};
\node (j)    at (-.5,-1) {$j$};
\node (j+1)    at (0,-1) {$j+1$};
\node (j2) at (2,-1) {$j$};
\node (j2+1) at (2.5,-1) {$j+1$};

\node (pi1) at  (-.5,-1.5) {$\varphi_l: \qquad\pi\in S_{(l,j)}^{(i,j+1)}$};
\node (pi2) at  (2.4,-1.5) {$\pi\in S_{(l,j+1)}^{(i,j)}$};
\node (pi2) at  (1,-1.5) {$\rightarrow$};


\draw[->, thick,black]
    (i) to[out=-90, in=90] 
    node[above,sloped,pos=0.2] {}
    (j+1);
\draw[->, thick,black]
    (l) to[out=-90, in=90] 
    node[above,sloped,pos=0.2] {}
    (j);
\draw[->, thick,black]
    (i2) to[out=-90, in=90] 
    node[above,sloped,pos=0.2] {}
    (j2);
\draw[->, thick,black]
    (l2) to[out=-90, in=90] 
    node[above,sloped,pos=0.2] {}
    (j2+1);
\end{tikzpicture}
    \caption{A schematic of the first mapping. }
    \label{fig:firstmap}
\end{figure}
\begin{proof}
Define the map $\varphi_\ell: S^{(i,j+1)}_{(\ell,j)} \,\to\, S^{(i,j)}_{(\ell,j+1)}$ by swapping the \emph{images} of $i$ and~$\ell$ in any permutation $\pi$ (see Figure~\ref{fig:firstmap}).  More precisely,
\[
  (\varphi_\ell \circ \pi)(i) \;=\;\pi(\ell), 
  \quad
  (\varphi_\ell  \circ\pi)(\ell)\;=\;\pi(i),
  \quad
  (\varphi_\ell  \circ\pi)(m)\;=\;\pi(m)\ \text{for }m\neq i,\ell.\]
By direct inspection,
    \begin{align}
       \log \frac{w(\pi)}{w(\varphi_l(\pi))} = &-\beta\left(|j+1-i|^\alpha +|j-l|^\alpha\right)
    +\beta\left(|j-i|^\alpha +|j+1-l|^\alpha \right)
       \\
       \leq& \left(-\beta (k+1)^\alpha +\beta(k)^\alpha \right) + 
       \left(-\beta|j-l|^\alpha + \beta|j+1-l|^\alpha\right).
    \end{align}
    A simple bounding argument shows that the second term $  \left(-\beta|j-l|^\alpha + \beta|j+1-l|^\alpha\right)$ is $\le 0$ if $\ell>j$. Since the upper bound is independent of $\pi$, by summing over all such $\pi\in S_{(\ell,j)}^{(i,j+1)}$ we finish the proof.
\end{proof}

\subsubsection{Handling the Problematic Cases with a Second Mapping}

\begin{figure}[ht]
\begin{minipage}[t]{0.22\linewidth}
    \centering
    \begin{tikzpicture}[%
      x=2.0cm,
      y=1.6cm,
      >=stealth,
      font=\small
    ]
    \node (neg)    at (-.5,0) {$\cdots$};
    \node (L)      at (0,0)   {$l$};
    \node (L0)     at (.5,0)  {$\cdots$};
    \node (Lj1)    at (1,0)   {$j+1$};
    \node (dots)   at (1.5,0) {$\dots$};
    \node (etan)   at (2,0)   {$\eta-1$};
    \node (Leta)   at (2.5,0) {$\eta$};
    \node[draw, rectangle,
          minimum width=2.3cm, 
          minimum height=1.0cm,
          label={[yshift=-.75cm]:{\(\,1\quad\dots\quad j-1\)}}]
          (LessJ) at (-0.5,-1) {};

    \node (J)   at (.5,-1) {$j$};
    \node (Jp1) at (1,-1)  {$j+1$};
    \node () at  (1.6,-1.8) {$\rho_l: \qquad\pi\in S_{(l,j+1)}^{(i,j)} \qquad \qquad\qquad\rightarrow$};

    \node[draw, rectangle,
          minimum width=2.3cm,
          minimum height=1.0cm,
          label={[yshift=-0.75cm]above:{\(\,j+2\quad\dots\quad n\)}}]
          (MoreJ) at (2,-1) {};

    \draw[->, thick, lightgray] (Lj1) to[out=-90, in=90] (1.5,-.7);
    \draw[->, thick, lightgray] (dots) to[out=-90, in=90] (MoreJ.north);
    \draw[->, thick, lightgray] (etan) to[out=-90, in=90] (2.5,-.7);

    \draw[->, thick, black] (Leta) to [out=-150, in=20](LessJ.north);
    \draw[->, thick, black] (L) to [out=-90, in=90] (.9,-.7);

    \end{tikzpicture}
\end{minipage}
\hspace{13em}
\begin{minipage}[t]{0.2\linewidth}
    \centering
    \begin{tikzpicture}[%
      x=2.0cm,
      y=1.6cm,
      >=stealth,
      font=\small
    ]
    \node (neg)    at (-.5,0) {$\cdots$};
    \node (L)      at (0,0)   {$l$};
    \node (L0)     at (.5,0)  {$\cdots$};
    \node (Lj1)    at (1,0)   {$j+1$};
    \node (dots)   at (1.5,0) {$\dots$};
    \node (etan)   at (2,0)   {$\eta-1$};
    \node (Leta)   at (2.5,0) {$\eta$};

    \node[draw, rectangle,
          minimum width=2.3cm, 
          minimum height=1.0cm,
          label={[yshift=-.75cm]:{\(\,1\quad\dots\quad j-1\)}}]
          (LessJ) at (-0.5,-1) {};

    \node (J)   at (.5,-1) {$j$};
    \node (Jp1) at (1,-1)  {$j+1$};
    
    \node () at  (.9,-1.8) {$\pi\in S_{(j^+,j+1)}^{(i,j)}$};

    \node[draw, rectangle,
          minimum width=2.3cm,
          minimum height=1.0cm,
          label={[yshift=-0.75cm]above:{\(\,j+2\quad\dots\quad n\)}}]
          (MoreJ) at (2,-1) {};

    \draw[->, thick, lightgray] (Lj1) to[out=-90, in=90] (1.5,-.7);
    \draw[->, thick, lightgray] (dots) to[out=-90, in=90] (MoreJ.north);
    \draw[->, thick, lightgray] (etan) to[out=-90, in=90] (2.5,-.7);

    \draw[->, thick, black] (Leta) to [out=-120, in=90](1,-.7);
    \draw[->, thick, black] (L) to [out=-90, in=90] (LessJ.north);

    \end{tikzpicture}
\end{minipage}
\caption{A schematic of the second mapping.}
\label{fig:secondmap}
\end{figure}
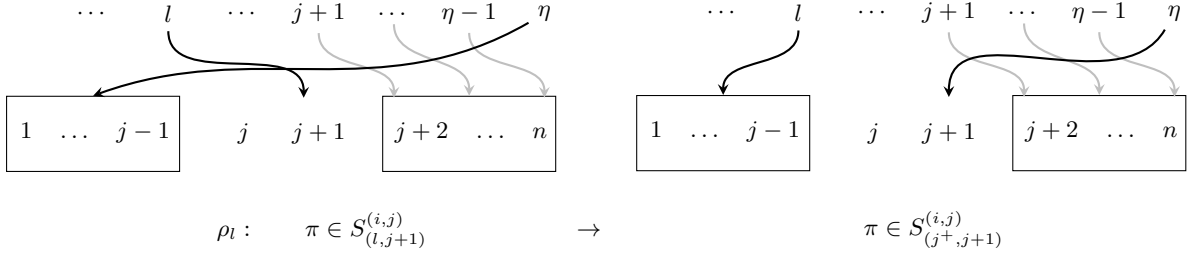

For $l\le j$, we handle those cases by constructing a second mapping that locates a suitable index that swapping it with $l$ will get a closer mapping to the central ranking (here identity).  See Figure~\ref{fig:secondmap}.
\begin{proposition}
\label{prop:second-mapping} Let $l\le j $. Also define, $ S_{(j^+,j+1)}^{(i,j)}=\bigcup_{k\ge j+1} S^{(i,j)}_{(k,j+1)}$. Then
\[  \frac{w(S^{(i,j)}_{(l,j+1)})}{ w(S^{(i,j)}_{(j^+,j+1)})}\leq e^{-\beta(j-l)^\alpha+\beta|j-1-l|^\alpha}.\]
\end{proposition}

\begin{proof}
Fix a permutation  $\pi \in S^{(i,j)}_{(\ell,j+1)}$. The first step is to find a mapping that maps $\pi$ to another permutation in $S^{(i,j)}_{(j^+,j+1)}$
For each permutation $\pi \in S^{(i,j)}_{(\ell,j+1)}$, there exists at least one index $\eta \in \{j+1,\dots,n\}$ such that $\pi(\eta) \le j-1$.  This is true 
because there are $n-j$ indices in $\{j+1,\dots,n\}$, but $\pi(\ell)=j+1$ forces at most $n-j-1$ of those indices to map above $j+1$.  
Define $\eta_\pi$ to be the smallest such index i.e., $  \eta_\pi \;:=\;
  \min\bigl\{\,
    m \in \{j+1,\dots,n\}:\;\pi(m)\le j-1
  \bigr\}$.

    Using this, we define the projection 
    $\rho_l: S^{(i,j)}_{(l,j+1)}\rightarrow S^{(i,j)}_{(j^+,j+1)},$ by swapping $(l,\eta_\pi)$ in the permutation $\pi$, i.e.,
    \[
  (\rho_\ell \circ \pi)(l) \;=\;\pi(\eta_\pi), 
  \quad
  (\rho_\ell  \circ\pi)(\eta_\pi)\;=\;\pi(l),
  \quad
  (\rho_\ell  \circ\pi)(m)\;=\;\pi(m)\ \text{for }m\neq \eta_\pi,\ell.\]

\paragraph{Injectivity.}
The map $\rho_l$ is one-to-one, but not onto. 
 In fact, for any $k\ge j+1$, any permutation $\pi'\in S^{(i,j)}_{(k,j+1)}$, can be uniquely inverted by swapping $l$ and the corresponding $k$ if it satisfies being in the image of $\rho_l$, i.e.,
    \begin{equation}\rho_l^{-1}(\pi') =
    \begin{cases}
             (l,k)\times \pi' & \text{ if } \pi'(l) \le j-1\text{ and }\pi'(s)> j+1 \text{ for }s\in\{j+1,j+2,\ldots,k\}\\
             null               & \text{ otherwise.}
    \end{cases}
    \end{equation}
    See Figure~\ref{fig:secondmap}.
Hence, summing over $\pi\in S^{(i,j)}_{(l,j+1)}$ and comparing to $\rho_l(\pi)$ in the codomain involves no collisions.

 \paragraph{Weight Change.}  
 Under the swap $\rho_k$, the weight $w(\pi)$ can change by a factor 
\[
   \frac{w(\pi)}{w(\rho_l(\pi))}=exp\big(-\beta\Bigl((j+1-l)^\alpha+(\eta_\pi-\pi(\eta_\pi))^\alpha-(\eta_\pi-j-1)^\alpha-
   |l-\pi(\eta_\pi)|^\alpha\Bigr)\big).
\]
Now, if we sum up over all possible values of $\eta_\pi\geq j+1$, we get
    \begin{align*}
        w(S^{(i,j)}_{(l,j+1)}) =& \sum_{k=j+1}^n 
        \sum_{\substack{\pi'\in S^{(i,j)}_{(k,j+1)}:\\ \rho_l^{-1}(\pi')\neq null}}
        w(\rho_l^{-1}(\pi'))\\
        =& \sum_{k=j+1}^n 
        \sum_{\substack{\pi'\in S^{(i,j)}_{(k,j+1)}:\\ \rho_l^{-1}(\pi')\neq null}}w(\pi')\exp\big(-\beta\Bigl((j+1-l)^\alpha+(k-\pi(k))^\alpha-(k-j-1)^\alpha\\
        &\qquad\qquad-
   |l-\pi(k)|^\alpha\Bigr)\big).
    \end{align*}
    Define $a(j,k,l,\pi(k)):=\exp\big(-\beta\Bigl((j+1-l)^\alpha+(k-\pi(k))^\alpha-(k-j-1)^\alpha-
   |l-\pi(k)|^\alpha\Bigr)\big)$.
Observe that $a(j,k,l,\pi(k))$  achieves its maximum when $\pi(k)$ achieves its highest value, i.e., $\pi(k)=j-1$. Now, substituting $\pi(k)=j-1$, the function $a(j,k,l,j-1)$ is decreasing in $k$, and since $k\ge j+1$,
\begin{equation}
    a(j,k,l,\pi(k)) \le a(j,j+1,l,j-1) = a(j,l),
\end{equation}
where $a(j,l):= \exp\left(
        -\beta(j+1-l)^\alpha
        +\beta|j-1-l|^\alpha\right)$.
  As a result, we prove the claim: \[  w(S^{(i,j)}_{(l,j+1)})\leq a(j,l)  \sum_{k=j+1}^n 
        \sum_{\substack{\pi'\in S^{(i,j)}_{(k,j+1)}:\\ \rho_l^{-1}(\pi')\neq null}}w(\pi')\leq a(j,l) w(S^{(i,j)}_{j^+,j+1}).\] 
 \end{proof}
\subsubsection{Geometric Decay Bound (Proof of Lemma~\ref{lm:Pij_decay})}\label{app:proof pij}
\begin{lemma}
\label{lm:Pij_decay_restate}
From Propositions~\ref{prop:first-mapping} and \ref{prop:second-mapping}, it follows that for sufficiently large $k$, for all $n$  and $j,i$ satisfying $j - i\geq k$, 
\[
\frac{P_n(i,\,j+1)}{P_n(i,\,j)} 
  \;\leq\;
  C(\alpha,\beta) < 1,
\] where $C(\alpha,\beta)=\frac{e^{-\beta\alpha k^{\alpha-1}}}{\alpha - 1} \frac{\Gamma\left(\frac{1}{\alpha - 1}\right)}{(\beta \alpha)^{\frac{1}{\alpha - 1}}}$  for $\alpha>1$, and $C(1,\beta)  = \frac{2e^{-2\beta}}{1+e^{-2\beta}}$
\end{lemma}

\begin{proof}
By applying Proposition~\ref{prop:first-mapping}, for $j-i\geq k$,
    \begin{align}\label{eq:wsj-}
        \frac{P_n(i,j+1)}{P_n(i,j)} \leq  
        &e^{-\beta (k+1)^\alpha +\beta(k)^\alpha }\frac{\sum_{l\in [n]} e^{-\beta|j-l|^\alpha + \beta|j+1-l|^\alpha}w(S^{(i,j)}_{(l,j+1)})}
        {\sum_{l\in[n]} w(S^{(i,j)}_{(l,j+1)})} \\\leq &
        e^{-\beta (k+1)^\alpha +\beta(k)^\alpha }\frac{w(S^{(i,j)}_{(j^+,j+1)}) + \sum_{l=1}^j e^{-\beta|j-l|^\alpha + \beta|j+1-l|^\alpha}w(S_{(l,j+1)}^{(i,j)}) }
        { w(S^{(i,j)}_{(j^+,j+1)}) + \sum_{l=1}^j w(S_{(l,j+1)}^{(i,j)})}\\
        = &
        e^{-\beta (k+1)^\alpha +\beta(k)^\alpha }\frac{1+ \sum_{l=1}^j e^{-\beta|j-l|^\alpha + \beta|j+1-l|^\alpha}\frac{w(S_{(l,j+1)}^{(i,j)}) }{w(S^{(i,j)}_{(j^+,j+1)}) }}
        {1+ \sum_{l=1}^j \frac{w(S_{(l,j+1)}^{(i,j)})}{w(S^{(i,j)}_{(j^+,j+1)}) }},
    \end{align}
    where in the second inequality we used the fact that for $l> j$, the coefficient  $e^{-\beta|j-l|^\alpha + \beta|j+1-l|^\alpha}$ is less than $1$.  
    
Now, we can use the second mapping (Proposition~\ref{prop:second-mapping}) to derive the desired upper bound. To see this, note that the coefficients in the numerator are greater than 1, in fact, for $l\leq j$, we have $exp(-\beta(j-l)^\alpha + \beta(j+1-l)^\alpha)\geq exp(\beta)$. Therefore, we can apply Proposition~\ref{prop:second-mapping},
 \begin{align*}
        \frac{P_n(i,j+1)}{P_n(i,j)} \leq &
        e^{-\beta (k+1)^\alpha +\beta(k)^\alpha }\frac{1+ \sum_{l=1}^j e^{-\beta|j-l|^\alpha + \beta|j+1-l|^\alpha}a(j,l)}
        {1+ \sum_{l=1}^j a(j,l)},
    \end{align*}
    where as before $a(j,l)= \exp\left(
        -\beta(j+1-l)^\alpha
        +\beta|j-1-l|^\alpha\right)$.
We now consider two cases based on the value of $\alpha$.

\paragraph{Case $\alpha>1$:}   In this case,
we can rewrite the numerator
\[\sum_{l=1}^j e^{-\beta|j-l|^\alpha + \beta|j+1-l|^\alpha}a(j,l)=e^{\beta}+\sum_{i=0}^{j-1}e^{-\beta((i+1)^{\alpha}-i^{\alpha})}\leq e^{\beta}+\sum_{i=0}^{j-1}e^{-\beta\alpha i^{\alpha-1}}. \]
For the denominator, we use the trivial lower bound of 1.
Therefore,
\[  \frac{P_n(i,j+1)}{P_n(i,j)} \le e^{-\beta (k+1)^\alpha +\beta(k)^\alpha }\Big(1+\big( e^{\beta}+\sum_{i=0}^{j-1}e^{-\beta\alpha i^{\alpha-1}}\big)\Big) .\]

For $\alpha>1$, the following limit  exists: $\lim_{j\to\infty}\sum_{i=0}^{j-1}e^{-\beta\alpha i^{\alpha-1}}=C_\infty(\alpha,\beta)$. Also, note that $C_\infty(\alpha,\beta)\leq \int_{0}^{\infty} e^{-\beta \alpha x^{\alpha - 1}} \, dx = \frac{1}{\alpha - 1} \frac{\Gamma\left(\frac{1}{\alpha - 1}\right)}{(\beta \alpha)^{\frac{1}{\alpha - 1}}}.
$ Also note that  $e^{-\beta((k+1)^\alpha-k^\alpha)}\leq e^{-\beta\alpha k^{\alpha-1}}$.
Therefore, 
$$\frac{P_n(i,j+1)}{P_n(i,j)}\leq e^{-\beta\alpha k^{\alpha-1}}C_\infty(\alpha,\beta).$$
 So, if we choose $k$ large enough we can see that $\frac{P_n(i,j+1)}{P_n(i,j)}$ is bounded by a constant smaller than 1. In fact,
$$\frac{P_n(i,j+1)}{P_n(i,j)}\leq 
\frac{e^{-\beta\alpha k^{\alpha-1}}}{\alpha - 1} \frac{\Gamma\left(\frac{1}{\alpha - 1}\right)}{(\beta \alpha)^{\frac{1}{\alpha - 1}}}.$$


\paragraph{Case $\alpha=1$:}    
By directly applying $\alpha=1$ in \eqref{eq:wsj-}, we get 
\begin{equation*}
        \frac{P_n(i,j+1)}{P_n(i,j)} \leq 
       e^{-\beta}\frac{e^{-\beta}w(S^{(i,j)}_{(j^+,j+1)}) + e^{\beta}w(S^{(i,j)}_{(j^-,j+1)}) }
        { w(S^{(i,j)}_{(j^+,j+1)}) + w(S^{(i,j)}_{(j^-,j+1)})}.
    \end{equation*}
    Now, note that for $\alpha=1$, the value of $a(j,l)$ is equal to $exp(-2\beta)$. Therefore, $w(S^{(i,j)}_{(j^-,j+1)})\leq e^{-2\beta}w(S^{(i,j)}_{(j^+,j+1)})$. Thus,
     \begin{align*}
        \frac{P_n(i,j+1)}{P_n(i,j)} &\leq 
       e^{-\beta}\Big(e^{-\beta}+\frac{(e^{\beta}-e^{-\beta})w(S^{(i,j)}_{(j^-,j+1)}) }
        { w(S^{(i,j)}_{(j^+,j+1)}) + w(S^{(i,j)}_{(j^-,j+1)})}\Big)\\
        & \leq 
       e^{-\beta}(e^{-\beta}+(e^{\beta}-e^{-\beta})\frac{e^{-2\beta}}
        { 1 + e^{-2\beta}})\\
        &=\frac{e^{-2\beta}(2-e^{-2\beta})}{1+e^{-2\beta}}\leq \frac{2e^{-2\beta}}{1+e^{-2\beta}}<1
    \end{align*}
Finishing the proof of the lemma in both cases.
\end{proof}

\subsection{Proof of Lemma~\ref{lm:tv_distance}}\label{sec: truncation proof}
We start by using Lemma~\ref{lm:Pij_decay} to prove the following claim.

\noindent\textbf{Claim}
Given the conditions of Theorem~\ref{thm:main sampling}, there exists constants $c<1$, \( k>0 \) and $\gamma$ such that for any \(i\in [n]\),
    \[
    \sum_{j : |j - i| > k} P_n(i, j) \leq \gamma c^{k}.
    \]
\begin{proof}
    Consider a fixed index \( i \).
    Let \( C(\alpha, \beta)<1 \) and $k$ be the constants from Lemma~\ref{lm:Pij_decay} such that for any \( j \) satisfying \( j - i > k\),
    \[
    P_n(i, j) \leq C(\alpha, \beta) P_n(i, j-1) \leq C(\alpha, \beta)^{|j - i|-k} P_n(i, i+k).
    \]
    A similar bound holds for $ P_n(i,j)$ when $j<i-k$. 
    
    To bound the sum of probabilities where the deviation exceeds \( k \), we split the summation into deviations to the right and deviations to the left of \( i \):
    \[
    \sum_{j : |j - i| > k} P_n(i, j) = \sum_{j > i + k} P_n(i, j) + \sum_{j < i - k} P_n(i, j).
    \]
    
    Applying the geometric decay bound to each term, we obtain
    \[
    \sum_{j : |j - i| > k} P_n(i, j) \leq \sum_{j > i + k} C(\alpha, \beta)^{j - i} \frac{P_n(i, i+k)}{C(\alpha, \beta)^{k}} + \sum_{j < i - k} C(\alpha, \beta)^{i - j} \frac{P_n(i, i-k)}{C(\alpha, \beta)^{k}} .
    \]
    Recognizing that both sums are identical in form, we can combine them:
    \[
    \sum_{j : |j - i| > k} P_n(i, j) \leq \Big(\frac{P_n(i, i-k)}{C(\alpha, \beta)^{k}} + \frac{P_n(i, i+k)}{C(\alpha, \beta)^{k}} \Big) \sum_{k = k + 1}^{\infty} C(\alpha, \beta)^k.
    \]    
    The infinite geometric series \( \sum_{k = k + 1}^{\infty} C(\alpha, \beta)^k \) converges to \( \frac{C(\alpha, \beta)^{k
    + 1}}{1 - C(\alpha, \beta)} \). Therefore, with $\gamma =2\frac{P_n(i, i+k)}{C(\alpha, \beta)^{k}(1 - C(\alpha, \beta))} $ and $c=C(\alpha,\beta)$, we have the proof of claim. 
\end{proof}
Now, we are ready to prove Lemma~\ref{lm:tv_distance}. 
\begin{proof}[Proof of Lemma~\ref{lm:tv_distance}]
We start with the first part.
Let $S_n^{(k)}$ be the set of permutations that maps at least 
one element further than $k$, i.e., $S^{(k)}_n = \{\pi\in S_n: \exists i\in[n]  \text{  s.t.  } |i-\pi(i)|> k\}$.  
    Note that for any $\pi \in S_{n}^{(k)}$, $ \P^{(k)}_{\alpha,\beta}(\pi) = 0\leq \P_{\alpha,\beta}(\pi)$, and for any $\pi\in S_n \setminus S_n^{(k)}$, $\P^{(k)}_{\alpha,\beta}(\pi)\geq \P_{\alpha,\beta}(\pi)$. Hence, the set of permutation that have a higher value over $\P$ is exactly $S_n^{(k)}$, i.e.,
    $$\{\pi\in S_n: \P_{\alpha,\beta}(\pi)>\P^{(k)}_{\alpha,\beta}(\pi)\}=S_n^{(k)}$$
    Therefore,
    by the definition of the total-variation distance,
    $$ \|\P_{\alpha,\beta} - \P^{(k)}_{\alpha,\beta}\|_{tv} = 2(\P_{\alpha,\beta}(S^{(k)}_n) - \P^{(k)}_{\alpha,\beta}(S^{(k)}_n)) = 2\P_{\alpha,\beta}(S^{(k)}_n) $$

    Define $S^{(k)}_{n,i} = \{\pi\in S_n: |i-\pi(i)|>k\}$. Therefore, $S^{(k)}_n = \bigcup_{i}S^{(k)}_{n,i}$. Let $c$ be the constants given by the claim above. Then if we choose $k=\frac{\log(2n/\epsilon)}{-\log(c)}$, by applying the result of the claim,
    $$\P_{\alpha,\beta}(S^{(k)}_{n,i}) \leq  \sum_{|j-i|>k} P_n(i,j)\leq \frac{\epsilon}{2n}$$
    As a result,
    $$\P_{\alpha,\beta}(S_n^{(k)})\leq \sum_{i=1}^n \P_{\alpha,\beta}(S^{(k)}_{n,i}) \leq n\frac{\epsilon}{2n}=\frac{\epsilon}{2}$$
    So, we conclude that $ \|\P_{\alpha,\beta} - \P^{(k)}_{\alpha,\beta}\|_{tv}\leq \epsilon$.

Now, we prove the bound on the permanent. 
Recall the function $w:S_n\rightarrow \mathbb{R^+}$ is the weight of each permutation:
    $$w(\pi) = \prod_{i=1}^n e^{-\beta|i-\pi(i)|^\alpha}.$$   Note that
    $$per(A_n) - per(A^{(k)}_n) = \sum_{\pi\in S^{(k)}_{n} } w(\pi) =per(A_n) \P_{\alpha,\beta}(S_n^{(k)}).$$
By applying total variation bound,
    $$\P_{\alpha,\beta}(S_n^{(k)})\leq \sum_{i=1}^n \P_{\alpha,\beta}(S^{(k)}_{n,i}) \leq n \frac{\epsilon}{n}\leq \epsilon,$$
    we get the second part of lemma:
    $$\frac{per(A_n) - per(A^{(k)}_n)}{per(A_n)} \leq \P_{\alpha,\beta}(S_n^{(k)})\leq \epsilon $$
\end{proof}

\subsection{Details of the Sampling Algorithm}\label{sec:dp construct}
In this section, we present the omitted details of how the sampling algorithm operates, explicitly using dynamic programming (DP).

\paragraph{DP states.}
We structure our DP algorithm into layers, where each layer corresponds to assigning a row  to a column. 
By assigning row $i$ to column $j$, we mean the permutation $\pi$ is constructed such that $\pi(i) = j$.
Specifically, we define DP states and transitions as follows:
\begin{itemize}
\item We have $n+1$ layers indexed by $i =  0, \dots, n$, representing stages in which rows $1, \dots, n$ are sequentially assigned to columns.
\item Each DP state at layer $i$ (for $i = 0, 1, \dots, n$) captures which columns are still available to assign at that stage. Due to the truncation, element $i$ can only be assigned to columns within a bandwidth of $\pm k$ around position $i$. Thus, each DP state is represented succinctly by a binary vector of length $2k$, where a '0' indicates an available column and a '1' indicates an assigned column.
\item The initial DP state at layer $-1$ is set to $\text{"1\dots10\dots0"}$, a binary string with $k$ ones followed by $k$ zeros. We initialize $DP[-1][\text{"1\dots10\dots0"}] = 1$.
\end{itemize}

\paragraph{DP Transitions and Updates.}
Transitions between DP states from layer $i-1$ to layer $i$ occur by:

\begin{enumerate}
\item \textbf{Left-shifting} the binary vector from the state at layer $i-1$, appending a new '0' to the right to indicate a new available column for row $i$.
\item \textbf{Flipping exactly one} of the '0's to '1', representing the assignment of the current element $i$ to a specific column. Precisely, assigning element $i$ to column $j$ corresponds to flipping the bit indexed by $i - j + k$. The edge's weight for this assignment is $\exp(-\beta|i-j|^\alpha)$.
\end{enumerate}

Each DP state at layer $i$ thus has exactly $k$ bits set to '1'. The cumulative weight of reaching state $s_{\text{new}}$ at layer $i$ is updated as:
\[
  DP[i][s_{\text{new}}] 
  \;=\; \sum_{s_{\text{old}}}\, 
  DP[i-1][s_{\text{old}}]
  \;\times\; \bigl(A_n^{(k)}[i,\, \text{col\_selected}(s_{\text{old}}\!\to\!s_{\text{new}})]\bigr),
\]
where the sum is over all valid predecessor states $s_{\text{old}}$.
Here, $\text{col\_selected}(s_{\text{old}} \to s_{\text{new}})$ denotes the bit flipped during the transition.

\paragraph{Sequential Sampling from the DP Table.}
The sampling begins from the terminal state at layer $n$, denoted by the binary vector $\text{"1\dots10\dots0"}$. At each step $i = n, n-1, \dots, 0$, we move backward through the DP table, selecting a predecessor state at layer $i-1$ with probability proportional to the corresponding DP state weight. The bit flipped during this transition identifies the column assigned to element $i$. This sequential backward pass continues until reaching the initial state at layer $-1$, thus completing the permutation.

The final permutation sampled this way precisely follows the truncated distribution $\hat{\P}_n^{(k)}$. Algorithms explicitly detailing these DP state constructions and sampling procedures are provided in Algorithms~\ref{alg:build_tag} and \ref{alg:sample_dp}.
Note the DP state space size at each layer is $\binom{2k}{k}$, with at most $k$ predecessors per state. Therefore, the forward DP computation and backward sampling both run in time $O\left(nk\binom{2k}{k}\right)$.

\subsection{Proof of  Lemma~\ref{thm:DP}}\label{sec:dp proof}
\begin{proof}
We proceed by induction on the layers, and we show  that $DP[i][s]$ is equal to the cumulative weights of all partial assignments up to  \(i\) given the availability of assignments corresponding to \(s\).

Base case $i=0$  holds by definition, since $DP[0][s_0] = 1.$ 
Assume that for some \(i \geq 0\), the DP table correctly stores the cumulative weights of all partial assignments up to layer \(i\) for every state \(s\) in layer \(i\). That is, for each state \(s\) in layer \(i\),
\[
DP[i][s] = \sum_{\pi \in S_i(s)} \prod_{j=1}^i A_n^{(k)}[j, \pi(j)],
\]
where \(S_i(s)\) denotes the set of all partial permutations up to layer \(i\) that result in state \(s\).

For the induction step, we need to show that the DP table correctly computes \(DP[i+1][s']\) for each state \(s'\) in layer \(i+1\).
Consider a state \(s'\) in layer \(i+1\). Let \(s\) be the predecessor state in layer \(i\), and let \(j\) be the column to which row \(i+1\) is assigned, corresponding to the flipped bit in \(s'\).
The cumulative weight for state \(s'\) is updated as:
\[
DP[i+1][s'] = \sum_{s \in \text{Pred}(s')} DP[i][s] \times A_n^{(k)}[i+1, j],
\]
where \(\text{Pred}(s')\) denotes all predecessor states \(s\) that can transition to \(s'\).

By induction hypothesis:
\[
DP[i][s] = \sum_{\pi \in S_i(s)} \prod_{j=1}^i A_n^{(k)}[j, \pi(j)].
\]
Substituting into the previous equation:
\[
DP[i+1][s'] = \sum_{s \in \text{Pred}(s')} \left( \sum_{\pi \in S_i(s)} \prod_{j=1}^i A_n^{(k)}[j, \pi(j)] \right) \times A_n^{(k)}[i+1, j].
\]
This can be rewritten as:
\[
DP[i+1][s'] = \sum_{s \in \text{Pred}(s')} \sum_{\pi \in S_i(s)} \prod_{j=1}^{i+1} A_n^{(k)}[j, \pi(j)],
\]
which is precisely:
\[
DP[i+1][s'] = \sum_{\pi \in S_{i+1}(s')} \prod_{j=1}^{i+1} A_n^{(k)}[j, \pi(j)],
\]
where \(S_{i+1}(s')\) denotes the set of all partial permutations up to layer \(i+1\) that result in state \(s'\).
Thus, the DP table correctly accumulates the weights for all states in layer \(i+1\).

After processing all \(n\) layers, the final layer \(n\) contains the terminal state \(s_n = ``1\cdots1\,0\cdots0"\). The value \(DP[n][s_n]\) equals \(\text{per}(A_n^{(k)})\), as it sums over all valid permutations \(\pi \in S_n\) weighted by \(\prod_{j=1}^n A_n^{(k)}[j, \pi(j)]\).

\textbf{Sampling Correctness:}
The sampling procedure performs a backward traversal from the terminal state \(s_n\) to the initial state \(s_0\). At each step \(i = n, n-1, \dots, 1\), the algorithm selects a predecessor state \(s\) with probability proportional to the weight of the edge connecting \(s\) to \(s'\). Formally, the probability of transitioning from \(s\) to \(s'\) is:
\[
P(s \rightarrow s') = \frac{ DP[i-1][s'] A_n^{(k)}[i, j]}{DP[i][s]},
\]
where \(j\) is the column assigned at this transition.
By the inductive construction of the DP table, these transition probabilities ensure that the probability of sampling a path
corresponding to\(\pi\) is:
\[\prod_{i=0}^n\frac{ DP[i-1][s'] A_n^{(k)}[i, \pi(i)]}{DP[i][s]}=\frac{\prod_{j=1}^n A_n^{(k)}[j, \pi(j)]}{DP[n][s_0]}=\frac{\prod_{j=1}^n A_n^{(k)}[j, \pi(j)]}{\text{per}(A_n^{(k)})} = \hat{\P}_{\alpha,\beta}^{(k)}(\pi).
\]
Therefore, the algorithm correctly samples permutations according to \(\hat{\P}_{\alpha,\beta}^{(k)}\).

\textbf{Time Complexity:}
The algorithm involves two main phases: building the DP table and performing the backward sampling. There are $n$ layers and  \(\binom{2k}{k}\) nodes per layer. Each state has at most \(k\) predecessors (since exactly one '0' is flipped to '1' within a window of \(2k\) positions), and each transition involved constant time multiplication and addition.     
    Therefore, the total time to build the DP table is:
    \[
    O\left(n \cdot \binom{2k}{k} \cdot k\right).
    \]
    
    Using the approximation \(\binom{2k}{k} \leq \frac{4^{(k)}}{\sqrt{\pi k}}\) (from Stirling's formula), the time complexity becomes:
    \[
    O\left(n \cdot \frac{4^k}{\sqrt{k}} \cdot k\right) = O\left(n \cdot 4^k \cdot \sqrt{k}\right).
    \]
    
After the DP is constructed, Algorithm~\ref{alg:sample_dp} generates samples from \(\hat{\P}_{\alpha,\beta}^{(k)}\) in time \(O(nk)\).
Using Lemma~\ref{lm:tv_distance}, to achieve $\epsilon$ total variation distance with true Mallows we need to set $k=\frac{\log(n/\epsilon)}{\log(1/C(\alpha,\beta))}$ where $C(\alpha,\beta)<\frac{2exp(-2\beta)}{1+\exp(-2\beta)}$ is given by Lemma~\ref{lm:Pij_decay}. So the runtime of computing permanent and preprocessing for sampling is  with $\epsilon$ error is $O\Big(n^{1-\frac{2}{\log(C(\alpha,\beta))}}\epsilon^{\frac{2}{\log(C(\alpha,\beta))}}\log(n/\epsilon)\Big)$
\end{proof}

\subsection{Pseudo-code: Sampling From Mallows Model}

\begin{algorithm}[ht!]
\caption{Build DP Table for Sampling}
\label{alg:build_tag}
\SetKwInOut{Input}{Input}
\SetKwInOut{Output}{Output}

\Input{$n,k \in \mathbb{N}$; Truncated matrix $A_n^{(k)}$ with bandwidth $2k$.}
\Output{DP States}

\BlankLine
$\mathit{DP} \gets$ array of dictionaries indexed by $\ell = -1, 0, \dots, n$\;
Define $\texttt{init\_state} \gets \underbrace{11\cdots1}_{k}\underbrace{00\cdots0}_{k}$ \;
$\mathit{DP}[-1][\texttt{init\_state}] \gets \varnothing$
\tcp*{Initialization layer with one node.}
\BlankLine
\For{$\ell \gets 0$ \KwTo $n$}{
    \ForEach{\textbf{state} $s_{\text{old}}$ in $\mathit{DP}[\ell-1]$}{
    $s_{\text{shift}} \gets s_{\text{old}}[1:] + \texttt{0}$
\tcp*{Remove leftmost bit and append a `0'.}\;
        \For{$i \gets 0$ \KwTo $2k-1$}{
            \If{$s_{\text{shift}}[i] = \texttt{0}$}{
                $s_{\text{new}} \gets \text{FlipBit}(s_{\text{shift}}, i)$\;
                $j \gets \ell - (i - k)$ 
                \tcp*{Column $j$ assigned to row $\ell$.}
                $w_{\text{edge}} \gets A_n^{(k)}[\ell, j]$ 
                \tcp*{Edge weight.}
                \If{$s_{\text{new}} \notin \mathit{DP}[\ell]$}{
                    $\mathit{DP}[\ell][s_{\text{new}}] \gets \varnothing$\;
                }
                Append $(s_{\text{new}}, w_{\text{edge}})$ to $\mathit{DP}[\ell-1][s_{\text{old}}]$\;
            }
        }
    }
}

\Return $\mathit{DP}$
\end{algorithm}

\begin{algorithm}[ht!]
\caption{Sample a Permutation from the DP Table}
\label{alg:sample_dp}
\SetKwInOut{Input}{Input}
\SetKwInOut{Output}{Output}

\Input{
       DP table $\mathrm{DP}[\ell][s]$ containing cumulative weights.}
\Output{Sampled permutation $\pi \in S_n$.}

\BlankLine
Initialize $\pi[1..n]$ as an empty array\;
$s_{\text{current}} \gets \underbrace{11\cdots1}_{k}\underbrace{00\cdots0}_{k}$ 
\tcp*{Start from final state at layer $n$.}

\For{$\ell \gets n$ \KwTo $1$}{
    $\mathit{predList} \gets \varnothing$ 
    \tcp*{Collect predecessors of $s_{\text{current}}$.}
    \ForEach{$(s_{\text{old}}, w_{\text{edge}})$ in $\mathit{DP}[\ell-1][\cdot]$}{
        \If{$(s_{\text{old}} \to s_{\text{current}})$ is valid}{
            Let $j \gets$ column assigned in this transition\;
            $\mathit{score} \gets \mathrm{DP}[\ell-1][s_{\text{old}}] \times w_{\text{edge}}$\;
            Append $(s_{\text{old}}, j, \mathit{score})$ to $\mathit{predList}$\;
        }
    }
    Normalize scores in $\mathit{predList}$ to form a probability distribution\;
    Randomly select $(s_{\text{old}}, j, \mathit{score})$ w.r.t.\ these probabilities\;
    $\pi[\ell] \gets j$ 
    \tcp*{Assign column $j$ to row $\ell$.}
    $s_{\text{current}} \gets s_{\text{old}}$\;
}

\Return $\pi$
\end{algorithm}

\pagebreak
\section{Additional Implementation Details}\label{exp-app}

\paragraph{Evaluation metrics.}
We evaluate predictive accuracy using the following metrics, comparing the test ranking $\pi$ and the predicted ranking $\sigma$:

\begin{enumerate}
    \item \textbf{Top-1/Top-5 Hit Rate:} Probability the top-ranked item in $\pi$ appears in the top-1 or top-5 positions of $\sigma$:
    \begin{equation}
        \text{Top-$k$ Hit rate} = \Pr\big(\pi(1)\in\{\sigma(1),\cdots, \sigma(k)\}\big).
    \end{equation}
    
    \item \textbf{Spearman's $\rho$:} Correlation based on squared rank differences:
    \[
    \rho = 1 - \frac{6 \sum_{i=1}^{n}( \pi(i)-\sigma(i) )^2}{n(n^2 - 1)}.
    \]
    
    \item \textbf{Kendall's $\tau$:} Correlation based on pairwise agreements as defined by $d_\tau$ in Section~\ref{sec:empiric}.
    
    \item \textbf{Hamming Distance:} Fraction of positions with differing ranks:
    \[
    H = \frac{1}{n}\sum_{i=1}^{n}\mathbf{1}\{\pi(i)\neq \sigma(i)\}.
    \]
    
    \item \textbf{Pairwise Accuracy:} Fraction of item pairs with matching relative orders:
    \[
    \text{Pairwise Accuracy} = \frac{\sum_{i<j}\mathbf{1}\{(\pi(i)-\pi(j))(\sigma(i)-\sigma(j))>0\}}{\binom{n}{2}}.
    \]
\end{enumerate}
In practice, all evaluation metrics are computed using Monte Carlo methods. We generate 100 predicted ranking samples from each trained model and estimate the metrics empirically. These empirical estimates approximate the true population-level metrics.


\paragraph{Baseline Implementation Details}
As mentioned in~\ref{sec:experimetins_MLE}, we compare against the Plackett-Luce (PL) and Mallow's $\tau$ model.

We fit the PL model via maximum likelihood estimation:
\begin{equation}
\label{eq:pl_loglik}
\mathcal{L}(\theta; \pi^{(1)},\ldots,\pi^{(m)}) = \sum_{\ell=1}^{m} \sum_{i=1}^{n} \left[ \theta_{\pi^{(\ell)}(i)} - \log \left( \sum_{j=i}^{n} \exp(\theta_{\pi^{(\ell)}(j)}) \right) \right].
\end{equation}
We optimize the negative log-likelihood using the L-BFGS algorithm \cite{liu1989limited}.  Additionally, we provide a fast sampler for the PL model using the Gumbel--Max \citep{yellott1977relationship, gumbelmax} to generate synthetic permutations:
\begin{equation}
\label{eq:gumbel}
\pi = \arg sort\left(-(\theta + \epsilon)\right), \quad \epsilon \sim \mathrm{Gumbel}(0,1)^n,
\end{equation}
where sorting is done in descending order to reflect best-to-worst preferences.

To fit the Mallows $\tau$ model, note that its  normalizer has a closed form \cite{fligner1986distance}:
\begin{equation*}
\label{eq:z_theta}
Z(\beta) = \prod_{j=1}^{n} \frac{1 - e^{-j \beta}}{1 - e^{-\beta}}.
\end{equation*}
Again, we do maximum likelihood estimation where the log-likelihood of the Mallows $\tau$ model is:
\begin{equation*}
\label{eq:mallows_loglik}
\mathcal{L}(\sigma, \beta; \pi^{(1)},\ldots,\pi^{(m)}) = -\beta \sum_{\ell=1}^{m} d_\tau(\pi^{(\ell)}, \sigma) - m \log Z(\beta).
\end{equation*}

We perform maximum likelihood estimation by:
\begin{itemize}
    \item Setting the central ranking \(\sigma_0\) to the Borda count aggregation over training rankings.
    \item Optimizing \(\beta\) using numerical minimization of the negative log-likelihood.
\end{itemize}
To generate samples from the Mallows--\(\tau\) model, we use the exponential insertion algorithm by \citep{lu2014effective}. At each step, an item from \(\sigma_0\) is inserted into a position within the partial permutation with probability proportional to \(\exp(-\beta \cdot \text{reverse\_index})\), favoring positions near the end. This yields efficient forward sampling in \(O(n^2)\) time.



\end{document}